\documentclass[11pt]{article}

\usepackage[final]{acl}

\usepackage{times}
\usepackage{latexsym}
\usepackage{amsmath}
\usepackage{subcaption}
\usepackage{amsthm}
\usepackage[T1]{fontenc}
\usepackage[utf8]{inputenc}

\usepackage{microtype}

\usepackage{inconsolata}

\usepackage{graphicx}
\usepackage{xspace}
\usepackage{amsfonts}
\usepackage[capitalise]{cleveref}
\usepackage{booktabs}
\usepackage{multirow}
\usepackage{tcolorbox}
\usepackage{dsfont}
\usepackage{algorithm}
\usepackage{algpseudocode}

\newtheorem{theorem}{Theorem}[section]

\newtheorem{proposition}[theorem]{Proposition}

\crefname{subsection}{Appendix}{Appendices}

\title{Select, Label, Evaluate: Active Testing in NLP}

\author{
 \textbf{Antonio Purificato\textsuperscript{1,2}},
 \textbf{Maria Sofia Bucarelli\textsuperscript{3,$\star$}},
 \textbf{Andrea Bacciu\textsuperscript{1}},\\
 \textbf{Fabrizio Silvestri\textsuperscript{2}},
 \textbf{Amin Mantrach\textsuperscript{1}}
\\
\\
 \textsuperscript{1} Amazon,
 \textsuperscript{2} Sapienza University of Rome,
 \textsuperscript{3}Université Côte d’Azur, CNRS, Inria, I3S
\\
 \small{
   \textbf{Correspondence:} \href{mailto:email@domain}{purifian@amazon.com}
 }
}

\algrenewcommand\algorithmiccomment[1]{\# #1}

\newcommand{\ndatasets}{18\xspace}
\newcommand{\nembeddings}{4\xspace}
\newcommand{\ntasks}{4\xspace}
\newcommand{\eg}{\textit{e.g.}\xspace}
\newcommand{\at}{Active Testing\xspace}
\newcommand{\expected}{\mathbb{E}}

\usepackage{mathtools}
\usepackage{float}
\newfloat{pipeline}{tbp}{lop}
\floatname{pipeline}{Pipeline}
\crefname{pipeline}{Pipeline}{Pipelines}
\Crefname{pipeline}{Pipeline}{Pipelines}

\begin{document}
\maketitle
\begingroup
\renewcommand{\thefootnote}{}
\footnotetext{$^{\star}$ Work done while at Sapienza.}
\endgroup
\begin{abstract}
Human annotation cost and time remain significant bottlenecks in Natural Language Processing (NLP), with test data annotation being particularly expensive due to the stringent requirement for low-error and high-quality labels necessary for reliable model evaluation. Traditional approaches require annotating entire test sets, leading to substantial resource requirements. 
\at  is a framework that selects the most informative test samples for annotation. Given a labeling budget, it aims to choose the subset that best estimates model performance while minimizing cost and human effort.
In this work, we formalize \at in NLP and we conduct an extensive benchmarking of existing approaches across \ndatasets datasets and \nembeddings embedding strategies spanning \ntasks different NLP tasks. 
The experiments show annotation reductions of up to 95\%, with performance estimation accuracy difference from the full test set  within 1\%.
Our analysis reveals variations in method effectiveness across different data characteristics and task types, with no single approach emerging as universally superior.
Lastly, to address the limitation of requiring a predefined annotation budget in existing sample selection strategies, we introduce an adaptive stopping criterion that automatically determines the optimal number of samples. We release our code at \url{https://github.com/amazon-science/NLPActiveTesting}.
\end{abstract}

\section{Introduction}

\begin{figure}
    \centering
    \includegraphics[width=\linewidth]{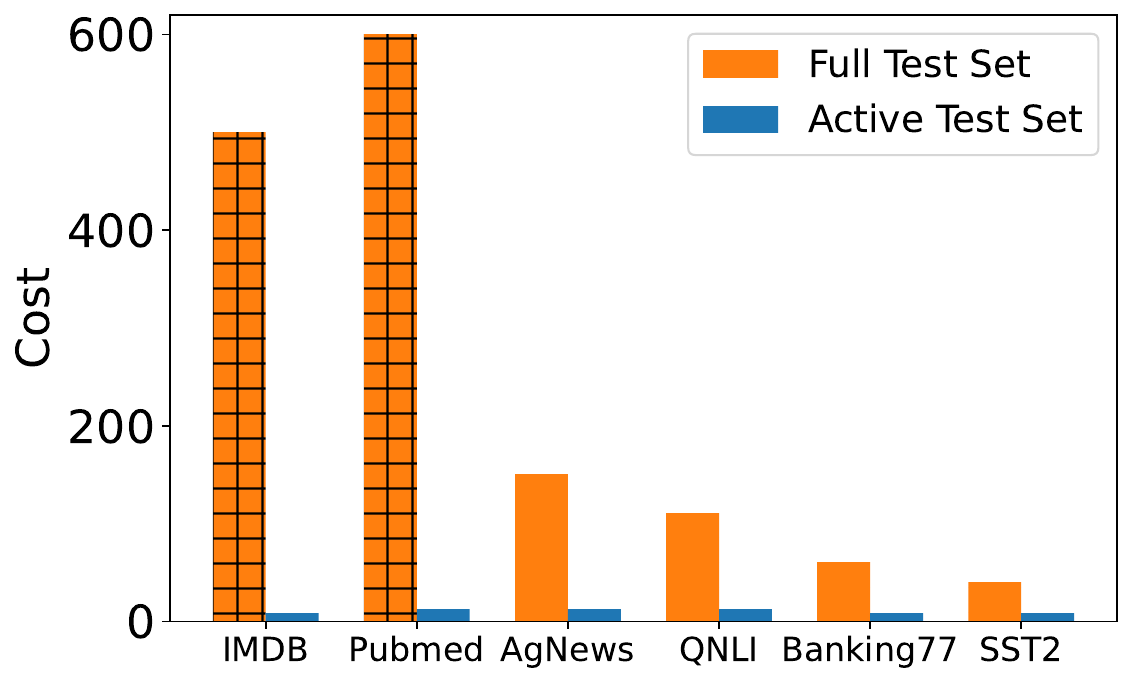}
    \caption{Annotation cost of using \at vs full test set ($\downarrow$ is better). Each sample has a cost of 0.02\$   (\href{https://labelyourdata.com/pricing}{LabelYourData}). The active cost is the one required to reach an estimation quality $<1\%$ wrt. full budget.}
    \label{fig:money_saved}
\end{figure}

Natural Language Processing (NLP) has witnessed unprecedented advances in recent years, with state-of-the-art models achieving remarkable performance across numerous tasks including text classification~\cite{liu2023zero} and question answering~\cite{parekh-etal-2025-dynamic}. However, the development and evaluation of these models continue to face a significant bottleneck: the acquisition of high-quality labeled data, particularly for test sets~\cite{kreutzer2022quality}. While training data can tolerate some level of noise, test sets require meticulous annotation with minimal error rates to serve as reliable benchmarks for model evaluation~\cite{gururangan-etal-2018-annotation}.
The conventional approach to model testing in NLP involves exhaustive annotation of entire test sets, representing a substantial investment of time and financial resources. This burden is acute in domains requiring specialized knowledge like the annotation of medical texts~\cite{wei2018clinical} or technical content often requires trained annotators, driving up costs and limiting scalability. Full test-set evaluation is often impractical: annotation costs vary notably across examples (\eg, in multilingual settings), and many samples may be redundant, with minimal impact on performance estimation.

In recent years, Large Language Models (LLMs) and AI agents have increasingly been used for data generation, annotation, and validation at scale~\cite{karim2025transforming}. While promising for reducing annotation effort, they raise concerns about quality, bias, and evaluation reliability, often still requiring human verification for test sets~\cite{zhen2025enhancing}.
\at addresses these challenges by selecting informative test samples using only unlabeled inputs, before annotation. Since not all test instances contribute equally to evaluation~\cite{kossen2021active}, \at aims to estimate model performance with statistical guarantees while minimizing annotation cost and preserving agreement with full test set evaluation.

This can be noted by looking at \cref{fig:money_saved}, which shows the amount of money saved by annotating using \at with respect to annotating the full test set. While related concepts have been explored in areas such as Active Learning~\cite{ahmadnia-etal-2025-active}, the systematic application of these principles to test data in NLP remains largely unexplored. %Moreover, most modern LLM evaluation methods, including LLM-as-a-judge, reduce open-ended tasks to classification or scoring. \at applies directly, enabling large annotation savings while preserving reliable performance estimates. 
\at differs from Active Learning in its objectives: while Active Learning aims to improve model training with fewer labeled examples, \at focuses on estimating test performance with minimal annotation, requiring different sampling strategies and metrics tailored to the testing context~\cite{bias}.
Specifically, our contributions are:
\begin{itemize}
    \item We present the first framework for \at in NLP, offering objectives and evaluation criteria and establishing a basis for research. Our results demonstrate annotation reductions of up to $\mathbf{95}\%$ while preserving performance estimate accuracy within $\mathbf{1}\%$ of testing on the full test set. We introduce an adaptive stopping criterion addressing a limitation of existing  \at methods: the need to pre-specify the annotation budget. 
    \item We show how \at operates in multilingual settings, optimizing cost by strategically allocating budget across languages.
    \item We show how \at can be applied for detection of samples from the minority class, helping annotators in identifying relevant samples. This is relevant in domains such as finance or e-commerce, where rare but high-impact events must be accurately identified under limited annotation budgets.
\end{itemize}

\section{Related Work}

Recent work in \at can be broadly into two areas: approaches focusing on optimal sampling under budget constraints, and methods designed for \at. However, budget-constrained sampling approaches, although sharing the goal of reducing annotation costs, are not comparable as they do not target test set optimization.

\subsection{Sampling with a limited budget}

\citet{zouhar2025select} investigate strategies for efficient subset selection in human evaluation of Natural Language Generation models, aiming to reduce annotation costs while preserving accurate rankings. They propose two frameworks: output-based methods, relying on model outputs and automatic metrics, and source-based methods, predicting item utility directly from inputs when outputs are unavailable. While \citet{zouhar2025select} focus on human evaluation scenarios, \citet{zhen2025enhancing} extend this line of research to automated evaluation, presenting an active sampling framework for automatic prompt optimization targeting LLM-as-a-judge systems. Their approach begins with zero labeled examples, iteratively selecting a sample set whose annotations are used to refine the prompt. Sample selection is posed as a convex optimization problem, balancing uncertainty and diversity to maximize utility under limited labeling budgets. Complementing these task-specific sampling strategies, \citet{nofree}  benchmark the influence of LLM embedding quality on deep Active Learning strategies. They evaluate five top-ranked embedding models from the Massive Text Embedding Benchmark (\href{https://huggingface.co/spaces/mteb/leaderboard}{MTEB}) leaderboard across diverse NLP text-classification datasets using multiple query strategies. 

While all three works share the goal of reducing annotation costs through intelligent sampling, they differ from our approach in a fundamental way: they optimize the training or evaluation pipeline itself (\eg, selecting examples for human evaluation or refining prompts), whereas our work focuses on constructing a minimal yet reliable test set for model assessment.

\subsection{\at}
The concept of \at was first introduced by \citet{nguyen2018activetestingefficientrobust} in the context of visual recognition systems trained on noisy datasets. Instead of annotating an entire test set, their framework selectively queries only a subset of examples to be re-annotated by humans. These annotated examples are then used to train a statistical estimator that improves the accuracy of performance metrics and actively guides which further examples should be annotated. Building on this foundation, \citet{kossen2021active} generalize the \at framework with a focus on reducing the number of test labels required for reliable model assessment. They propose actively selecting test points to label using acquisition strategies tailored to the testing phase. While such active selection introduces sampling bias, based on the work of \citet{bias}, they derive estimators that are unbiased and exhibit reduced variance, leading to more accurate evaluation with fewer labels. They demonstrate that, across different image classification models, \at can match the accuracy of \textit{i.i.d.} evaluation with significantly fewer labeled examples. 

Although both works establish foundations for \at, they remain confined to the computer vision domain and rely on specific heuristics. In contrast, our work provides the first framework for \at in NLP, accompanied by empirical benchmarking and an adaptive method that generalizes beyond prior approaches.

\citet{pmlr-v267-li25bp,zhang2025how} share our goal of reducing the cost of model evaluation by selecting a subset of test examples while maintaining reliable estimates of overall performance. In particular, \citet{pmlr-v267-li25bp} reduces the number of benchmark prompts that need to be evaluated by modeling dependencies among evaluation examples and actively selecting prompts that are informative for estimating overall benchmark performance. Similarly, \citet{zhang2025how} addresses benchmark prediction by leveraging the historical performance of previously evaluated models to predict the performance of new models from a small subset of benchmark examples. While these approaches are closely related to our objective, they differ in scope and methodology from our setting, where the focus is on active testing and sample selection for unbiased performance estimation, rather than on reducing benchmark prompts through dependency modeling or predicting model performance from historical evaluation data.

\section{Method}
Given a set of samples $X$ with $|X|=N$, we define $X_A$ as the set of annotated samples and $X_{NA}$ as the set of non-annotated samples $(X=X_A \cup X_{NA})$. We use the same notation for the labels $Y=Y_A \cup Y_{NA}$. We have a model $f:X \rightarrow Y$.
An \at algorithm has a budget $B$ and each query could have a different cost. Given a metric $M=M(f(X),Y)$, we define the estimation error as:
\begin{equation*}
    E(M) = \mid M(X_F) - M(X_A) \mid
\end{equation*}

The estimation error is a dimensionless quantity which measures the difference between the metric computed on the fully annotated test set ($X_F$) and on the actively selected subset ($X_A$). Throughout the paper, we use an unbiased estimator for the latter, presented in the next paragraph. We focus on performance metrics commonly used by practitioners like accuracy, precision, recall (for classification, NER and POS tagging)~\cite{liu2023zero} and ROUGE (for summarization)~\cite{li2024active} as these metrics, differently from loss, provide a reliable way to evaluate model performance. Given cost $C(x)$ of annotating sample $x$, we aim to find:
\begin{align*}
    \min E(M) \text{ s.t.} \\
    \underset{x \in X_A}{\sum} C(x) \leq B
\end{align*}

We seek to select the subset $X_A$ minimizing the estimation error while ensuring that the total annotation cost does not exceed the budget $B$.

\subsection{Unbiased \at}

The simplest and most used approach to handle limited annotation budgets is uniform random sampling. Given a labeling budget $B$, it selects a subset of examples uniformly at random and, given a metric $M$, it is defined as:

\begin{equation*}
\widehat{M}_{\text {random}}=\frac{1}{B} \sum_{i=1}^B M(f(x_i), y_i)
\label{eq:ht_estimator}
\end{equation*}

In uniform sampling, the random estimator is unbiased $\left(\expected[\widehat{M}_{\text {random}}] = M\right)$ and it converges to the true empirical metric as $B \rightarrow N$, but it can produce estimates with large variance. 
%However, when working with small annotation budgets ($B << N$), random sampling can produce estimates with large variance.
%This means that while the estimate is correct on average, any single run of this method might yield results that deviate substantially from the true metric.

As pointed out in~\citet{bias,kossen2021active}, \at introduces a bias into our estimates, since the samples are not drawn from the population distribution.
To solve this issue we use the Inverse Probability Weighted Estimator~\cite{horvitz1952generalization} that  provides an unbiased estimator which is defined as:
\begin{equation*}
     \widehat{M} = \frac{1}{B} \sum_{i=1}^{B} \frac{M(f(x_i),y_i)}{Nq_i}
\end{equation*}

Where $q_i = q(x_i; X_{1:i-1}, X)$ is the probability mass 
for datum $x_i$ of being the next to be sampled and it depends on the \at strategy. Additional details and derivations for our estimator are provided in \cref{app:unbiasedness}.
During the rest of the paper we focus on accuracy rather than loss as the primary evaluation metric but the same estimators for precision, recall can be found in \cref{app:unbiased_est}. Accuracy provides an interpretable, bounded measure of performance that directly reflects classification quality and it is the metrics used in related works~\cite{kossen2021active}. For summarization we focus on ROUGE.
Estimators like PURE and LURE~\cite{bias} cannot be used in this case, since they tend to overestimate the value of the accuracy and they do not satisfy some properties of the accuracy like being in $[0,1]$ (see \cref{app:limits_estimator}). 
This limitation is problematic when evaluating accuracy, where boundedness is an essential property for interpretation. 
\subsection{\at Strategies}
We evaluate \at strategies for textual data. Although some are adapted from Active Learning, the two address different problems, as discussed in~\citet{kossen2021active} and in~\cref{app:al_vs_at}. Additional details on the \at strategies can be found in~\cref{app:at_details}.

\begin{itemize}
    \item Surrogate [RF, SVM]~\cite{kossen2021active}:
  Selects samples using uncertainty estimated by a surrogate classifier trained on embeddings and accumulated labeled data. Supported models are a 300-tree Random Forest and an RBF-SVM. Uncertainty is measured as Shannon entropy over calibrated class probabilities, and samples with highest normalized entropy are selected. A limitation is the need to progressively collect ground-truth labels to train the surrogate.
    \item Uncertainty [GP, MI]~\cite{kossen2021active}:
  Estimates uncertainty via Monte Carlo Dropout on the pretrained embedder (not the evaluated LLM), using ten forward passes. MI computes the difference between predictive entropy and average per-pass entropy, while GP measures latent-space variability through embedding standard deviations across passes.
    \item Coverage~\cite{maharanamathbb}:
  Maximizes geometric diversity in embedding space. Starting from a random sample, it iteratively selects the point farthest (Euclidean distance) from the last chosen sample, promoting broad exploration without clustering assumptions.
    \item Stratified Random: Samples proportionally from each class to preserve the original label distribution. However, it is not suitable for \at, as it requires prior knowledge of class distributions in the unlabeled dataset.
    \item Agreement (ours):
The Agreement strategy ranks unlabeled samples by the level of disagreement among attention heads. Text inputs are first encoded using a pretrained language model,  producing contextual token representations. These are then passed through a separate, lightweight multi-head self-attention layer with a fixed number of 8 heads, independent of the LLM predictor being evaluated, which outputs both transformed token embeddings and per-head attention weight matrices. For each sample, the uncertainty score is computed by taking the variance of the attention weights. More precisely, for each token position $t$, the vector containing all token-to-token attention scores for head $h$ is considered, and its variance across heads is computed. The resulting variance matrices are averaged over all token positions, producing a single scalar score for each sample. All unlabeled samples not previously selected by the algorithm are ranked in descending order according to their normalized attention-based scores.
\end{itemize}
\subsection{Stopping Criterion}
\begin{algorithm}[t]
\caption{Stopping Criterion}
\label{alg_stopping_criterion}
\begin{algorithmic}
\Require Dataset $\mathcal{D}$, budget $B$, threshold $\tau$
\For{$i$ in range($B$)}
\State Pick a sample $x_i$ using an \at strategy and annotate it $\rightarrow$ label $y_i$;
\State $X_A = X_A \cup x_i$ and $Y_A = Y_A \cup y_i$;
\State Compute the prediction $f(X_A)$;
\State Compute $\widehat{M}_{\text{random}}(f(X_A),Y_A)$;
\State Compute unbiased metric $\widehat{M}(f(X_A),Y_A)$;
\If{$ \mid \widehat{M}_{\text{random}}- \widehat{M} \mid < \tau $} \State \textbf{break}
\Else
    \State \textbf{continue}
\EndIf
\EndFor
\end{algorithmic}
\end{algorithm}

To the best of our knowledge, existing works on \at~\cite{kossen2021active} typically set a fixed budget $B$ of samples to be annotated. Therefore, the annotation process exhausts the whole budget, regardless of the actual need for it. However, in certain scenarios, this might be unnecessary. Since the actual number of samples to be annotated cannot be determined a priori, we propose \cref{alg_stopping_criterion}, which automatically terminates the annotation process as soon as we estimate that a sufficient number of samples have been annotated.
Since the algorithm terminates adaptively before the full test set is
annotated, $M(X_F)$ is unavailable. Therefore, in this setting, the
estimation error is computed as the difference between the raw metric
on $X_A$ and its unbiased estimate on the same subset.

The stopping criterion in \cref{alg_stopping_criterion} relies on the convergence between the unbiased estimator $\widehat{M}$ and the baseline random estimator $\widehat{M}_{\text{random}}$. Let $\widehat{M}^{(t)}$ denote the unbiased estimate after $t$ annotated samples. As \at progresses and the sampling strategy sufficiently explores the data distribution, the bias term $\mathbb{E}[\widehat{M}^{(t)} - \widehat{M}_{\text{random}}^{(t)}]$ converges to zero. When the empirical difference $|\widehat{M}^{(t)} - \widehat{M}_{\text{random}}^{(t)}| < \tau$ for a small threshold $\tau$, we deem the estimator stable—further annotations are unlikely to change the performance estimate meaningfully. The user-defined convergence threshold $\tau$ specifies how close the unbiased estimate ($\widehat{M}$) must be to the unweighted random-sampling estimate ($\widehat{M}_{\text{random}}$) before we consider the active set sufficiently representative to stop annotating. 
More formally, the following proposition holds:
\looseness=-1
\begin{proposition}
\label{prop_stopping_criterion_convergence}
    Let $\widehat{M}_{\text{random}}$ be the unweighted estimator of the metric and $\widehat{M}$ the inverse probability (Horvitz-Thompson) estimator defined in \cref{eq:ht_estimator}. If the sample is drawn without replacement and labeling budget $B \to N$ (approaching full annotation), then 
    $ \widehat{M}_{\text{random}} - \widehat{M} \xrightarrow{P} 0 $.
\end{proposition}

Proof can be found in \cref{app:proof_convergence}, where we also
derive an explicit finite-sample upper bound on the gap as a function of the annotation budget.

\label{sec:exp_setup} 

\begin{pipeline}[t]
\caption{\at Framework}
\label{alg_main}
\begin{algorithmic}
\Require Dataset $\mathcal{D}$, embedding strategy $E$, budget $B$, test set $X$;
\State Compute embeddings on $X$;
\State Compute model predictions $f(X)$; \Comment{\texttt{Stored for metric eval.}}
\State Given the embeddings, construct $X_A$ via the chosen \at strategy;
\State Evaluate predictions using $X_F$ and the active $X_A$ test sets and compute the metrics;
\end{algorithmic}
\end{pipeline}

\section{Experiments}
\subsection{General Framework}
\begin{table}[t]
  \centering
  \caption{Class distribution of the selected datasets and time required to compute the embeddings on the full datasets in minutes ($T$). Number of train and test samples and values of the budget are also shown.}
    \label{app:table_datasets} 
  \resizebox{\linewidth}{!}{
  \begin{tabular}{lcccccc}
        \toprule
        \textbf{Name} & \#Classes & Class Distribution (\%) & $T$ [minutes] & |Train| & |Test| & $B$ \\
        \midrule
        \href{https://huggingface.co/datasets/fancyzhx/ag_news}{AG's News}                       & 4         & [25,25,25,25] & 7 & 120,000 & 7,600 & 1,000     \\
        \href{https://huggingface.co/datasets/PolyAI/banking77}{Banking77} & 77  & - & 2 & 10,000 & 3,000 & 1,000 \\
        \href{https://huggingface.co/datasets/pietrolesci/dbpedia_14_indexed}{DBPedia}    & 14        & [7.1,7.1,7.1,7.1,7.1,7.1,7.1,7.1,7.1,7.1,7.1,7.1,7.1,7.1] & 77 & 560,000 & 5,000 & 1,000\\
        \href{https://huggingface.co/datasets/nid989/FNC-1}{FNC-1} & 4         & [6.72,1.50,17.77,74.01]  & 2 & 40,000 & 4,998 & 1,000   \\

        \href{https://huggingface.co/datasets/nyu-mll/glue}{QNLI}  & 2         & [49,51] & 3 & 104,000 & 5,463 & 1,000\\
        \href{https://huggingface.co/datasets/stanfordnlp/sst2}{SST-2} & 2         & [49,51] &  1 & 67,000 & 872 & 600 \\
        \href{https://huggingface.co/datasets/OxAISH-AL-LLM/trec6}{TREC-6} & 6         & [1.8,18.80, 27.60, 13.00, 16.2, 22.60]   & 1 & 5,452 & 500 & 400  \\
         \href{https://huggingface.co/datasets/stanfordnlp/imdb}{IMDB}$^*$   & 2         & [50,50] & 41 & 25,000 & 25,000 & 1,000     \\
         \href{https://huggingface.co/datasets/pietrolesci/pubmed-20k-rct}{PubMed}$^*$ & 5 & [10.4,15.45, 33.42,7.89,32.84] & 21 & 170,000 & 30,000 & 1,000\\
         \href{https://huggingface.co/datasets/dair-ai/emotion}{Emotion}$^*$ & 6 & [29.05,34.75,7.95, 13.75,11.20,3.30] & 2 & 16,000 & 2,000 & 1,000 \\
         \href{https://huggingface.co/datasets/cornell-movie-review-data/rotten_tomatoes}{Rotten}$^*$ & 2 & [50,50] & 1 & 8,000 & 1,000 & 600\\

         \href{https://huggingface.co/datasets/tyqiangz/multilingual-sentiments}{Multilingual}$^*$ & 3 & [33,33,34] & 1 & 1,840 & 880 & 600 \\
         \href{https://huggingface.co/datasets/Babelscape/wikineural}{WikiNeural}$^*$ & 9 & [87.7,1.9,1.5,1.3,1.0,2.2,0.9,1.7,1.8] & 6 & 92,720 & 11,579 & 1,000 \\
    \href{https://huggingface.co/datasets/universalner/universal_ner}{Universal NER}$^*$ & 9 & [93.3,1.8,0.9,1.3,1.1,1.2,0.3,0.05,0.05] & 1  & 12,543 & 2,077 & 1,000\\
         \href{https://huggingface.co/datasets/commul/universal_dependencies}{UD-ATIS}$^*$ & 17 & [17,21,2,3,1,7,1,23,5,1.5,0.5,10,4] & 1 & 4,274 & 586 & 500 \\
         \href{https://huggingface.co/datasets/commul/universal_dependencies}{UD-EWT}$^*$ & 17 & [17,12,8,3,0.4,1.5,7.1,3,7,3,8,8,1,1,4,10,6] & 1 & 12,544 & 2,077 & 1,000 \\
        \href{https://huggingface.co/datasets/abisee/cnn_dailymail}{CNN}$^*$ & NA & NA & 4 & 287,113 & 11,490 & 1,000 \\
         \href{https://huggingface.co/datasets/csebuetnlp/xlsum}{XLSum}$^*$ & NA & NA & 2 & 38,110 & 4,763 & 1,000 \\
        \bottomrule
    \end{tabular}
    }
\end{table}
\begin{figure}[t]
     \centering
    \includegraphics[width=\linewidth]{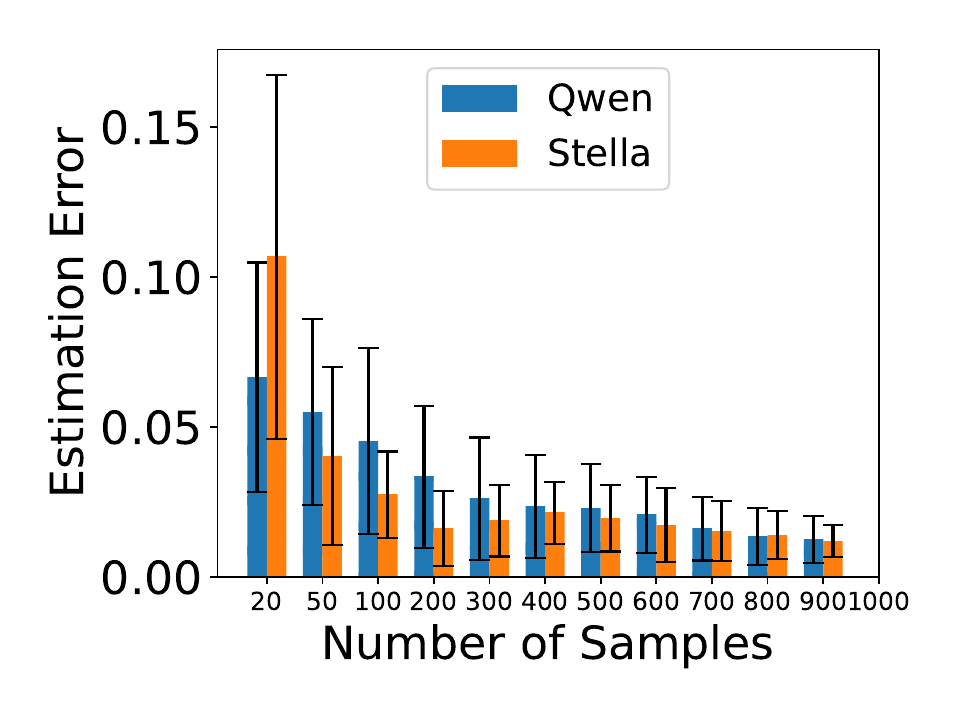}
    \caption{Estimation error with Agreement on AgNews using Qwen and Stella embedders ($\downarrow$ is better).}
    \label{fig:emb_difference}
\end{figure}

\begin{figure*}[t]
    \centering
\includegraphics[width=.7\linewidth]{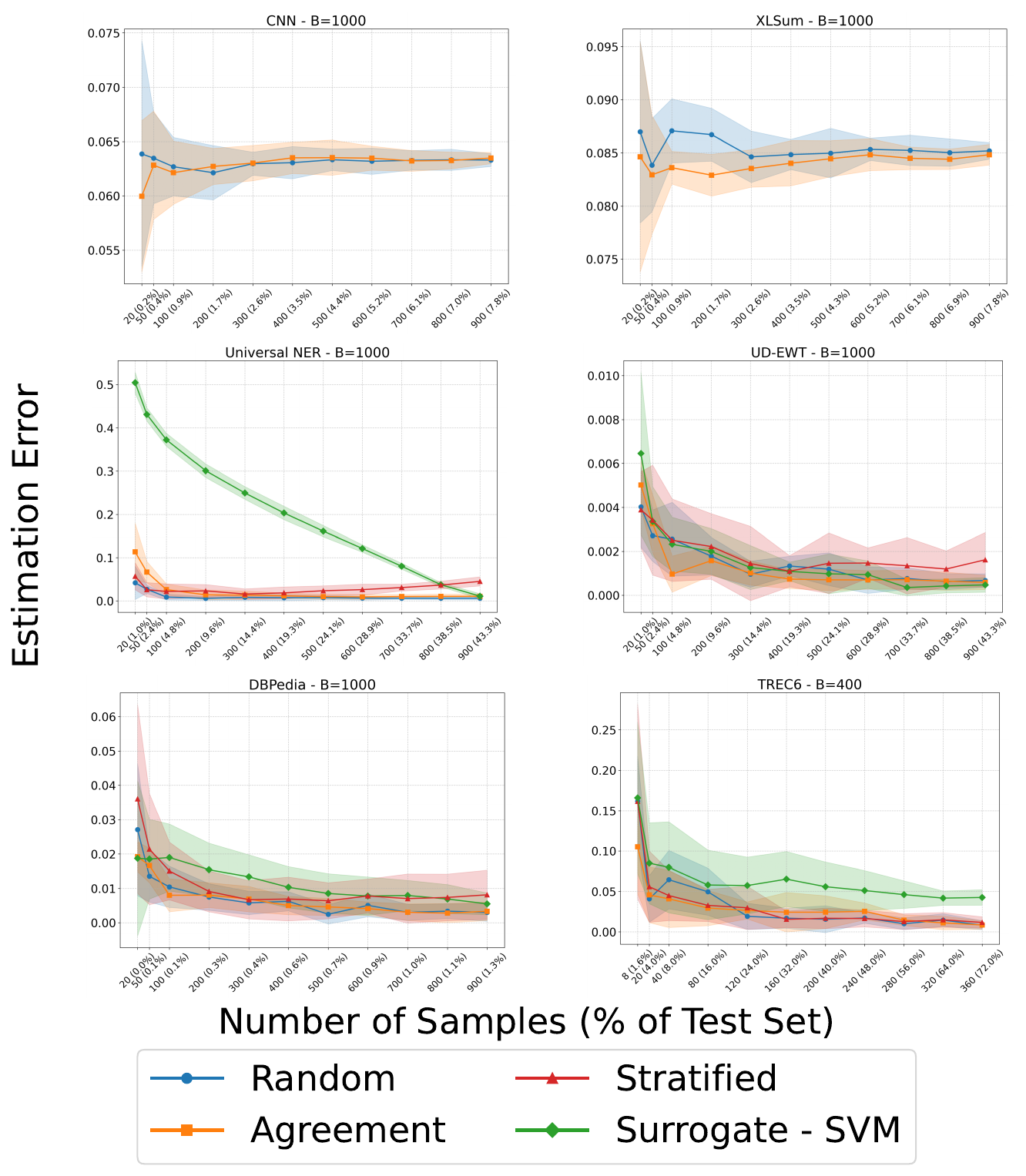}
        \caption{\at results in terms of accuracy on all the datasets by fixing the embedding strategy and the predictor. We observe that the Agreement strategy is performing better than Surrogate on the majority of the datasets. With a reduced budget Agreement is also able to beat Random ($\downarrow$ is better).}
    \label{fig:all_datasets}
\end{figure*}

We use the same experimental setup as~\citet{nofree}. Our framework follows~\cref{alg_main}. 
% In order to compute the estimation error, which is used to evaluate the results of \at, we need the predictions on the full dataset $X_F$. 
All strategies share a common structure: given the embedding space of $X$, they iteratively select samples to form the active set $X_A$ until the budget $B$ is exhausted or the stopping condition is met. At each iteration, an \at strategy assigns a score to each unlabeled sample, and the highest-scoring samples are added to $X_A$. 
We test 8 \at strategies on \ndatasets datasets using \nembeddings embedding strategies and 3 predictors. This results in a large number of experiments. Due to space constraints, we show the results on the selected \at strategies with one embedding strategy (Qwen) and one predictor (Claude 4.5). The results are averaged over 10 different seeds. Remaining plots can be found in our repository.

\paragraph{Hardware}
We run all the experiments on a machine with 96 AMD EPYC 7R13 Processor CPUs and 4 NVIDIA L40S GPUs with 48 GB of RAM.

\paragraph{Datasets}
We use \ndatasets datasets, covering text classification (12), POS tagging (2), NER (2) and summarization (2), presented in \cref{app:table_datasets}. We take all the data from the ActiveGLAE benchmark~\cite{activeglae} with the addition of more data (marked with *). We try to set, when possible, the same value of $B$ for all the datasets ($B=1000$). Additional statistics \footnote{\noindent For Banking77, values are omitted due to high number of classes, with the data being balanced ($\sim1.29\%$ of samples per class). NA indicates summarization datasets with no classes.} are presented in \cref{app:table_datasets}.

\paragraph{Embedding Strategies}
Since the quality of the embeddings influences the results of Active Learning~\cite{nofree}, we investigate whether this extends to \at. We evaluate \nembeddings different models from the MTEB leaderboard (\href{https://huggingface.co/google-bert/bert-base-uncased}{Bert}, \href{https://huggingface.co/docs/transformers/model_doc/distilbert}{DistilBert}, \href{https://huggingface.co/Qwen/Qwen3-0.6B}{Qwen}, \href{https://huggingface.co/NovaSearch/stella_en_1.5B_v5}{Stella}). As shown in \cref{fig:emb_difference}, even on the same dataset and with the same strategy, performance varies across embedding models. The effect of different embedding strategies and predictors on \at is analyzed in~\cref{app:embedding}.

\paragraph{Predictor}
To avoid bias from training, we use Claude Sonnet 4.5, Amazon Nova Pro and Qwen for the text-classification task. 
The selected predictors are able to demonstrate good performance in text-classification and summarization, as also shown in \cref{app:predictor_performance}, where we also present the prompts used for classification and summarization. For POS tagging, we use \href{https://huggingface.co/vblagoje/bert-english-uncased-finetuned-pos}{BERT} and \href{https://huggingface.co/jordigonzm/mdeberta-v3-base-multilingual-pos-tagger}{DeBERTa} finetuned, while for NER we use \href{https://huggingface.co/Babelscape/cner-base}{CNER} and \href{https://huggingface.co/guishe/nuner-v1_orgs}{UNER}.

\subsection{Multilingual Setup}
% We test how this framework works in the case of multilingual experiments. This is because annotating samples in some languages can be highly expensive, while other languages, which are more common, are cheaper. The goal of this experiment is to show how to leverage annotations from two languages and to merge them in order to optimize the budget. 
Annotation costs vary significantly across languages, with some (\eg, English) being relatively cheap and others considerably more expensive. We investigate whether \at can reduce annotation effort in high-cost languages by leveraging samples from low-cost languages, transferring evaluation knowledge across linguistic boundaries.

 \begin{figure*}[t]
    \centering
\includegraphics[width=0.7\linewidth]{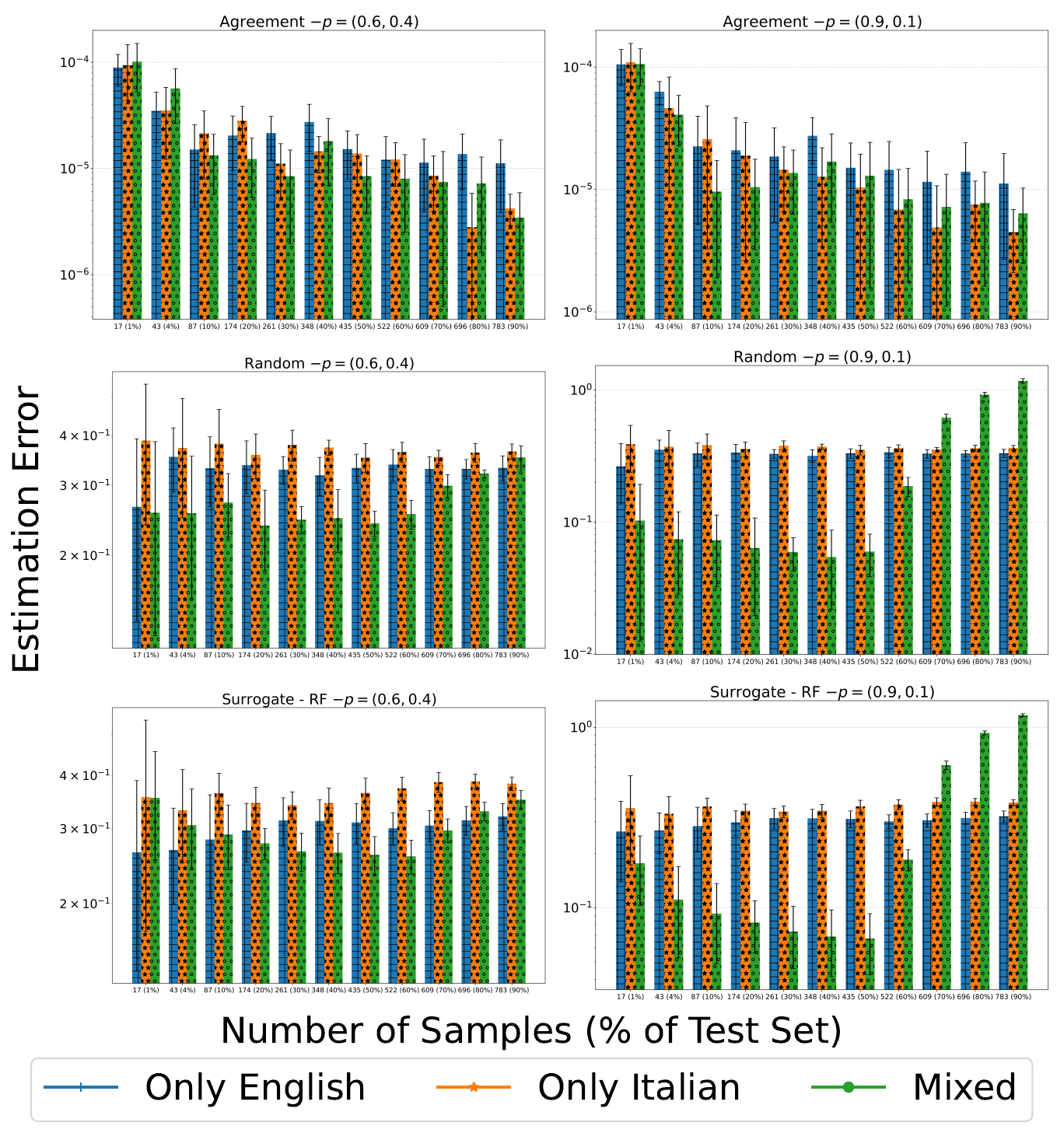}
        \caption{Multilingual results. Each row is a value of prior $p$. Each column refers to a method ($\downarrow$ is better).}
     \label{fig:multilingual}
\end{figure*}

For this experiment, we apply the steps presented in \cref{alg_main}. The difference is that dataset $\mathcal{D}$ contains samples from two languages and has a prior $p$ on the languages; in our case the prior reflects the relative cost of annotating texts in each language. The estimation error is computed with respect to the set containing the annotations for both languages.
This experiment assumes cross-lingual transferability, where a model performing well in one language achieves similar accuracy in the other. This is reasonable for multilingual models with shared vocabularies, as sentence-level classification tends to be language-agnostic~\cite{zhao2024adamergex}.

\subsection{Minority Class Sample Detection}
\at presents advantages for highly imbalanced datasets. This is relevant in domains such as finance or A/B testing, where rare but high-impact events must be accurately identified. Examples include detecting fraudulent transactions for risk management or identifying negative user feedback to maintain product quality. 
To investigate this aspect, we conduct experiments under various budget constraints $B$. We define ``minority samples'' as those belonging to the class with the fewest instances in the test set, and evaluate how effectively each strategy identifies and selects such samples. Specifically, we measure the number of minority-class samples included in the active set, providing insights into model behavior.

\section{Results} 
\begin{figure}[t]
    \centering
    \includegraphics[width=\linewidth]{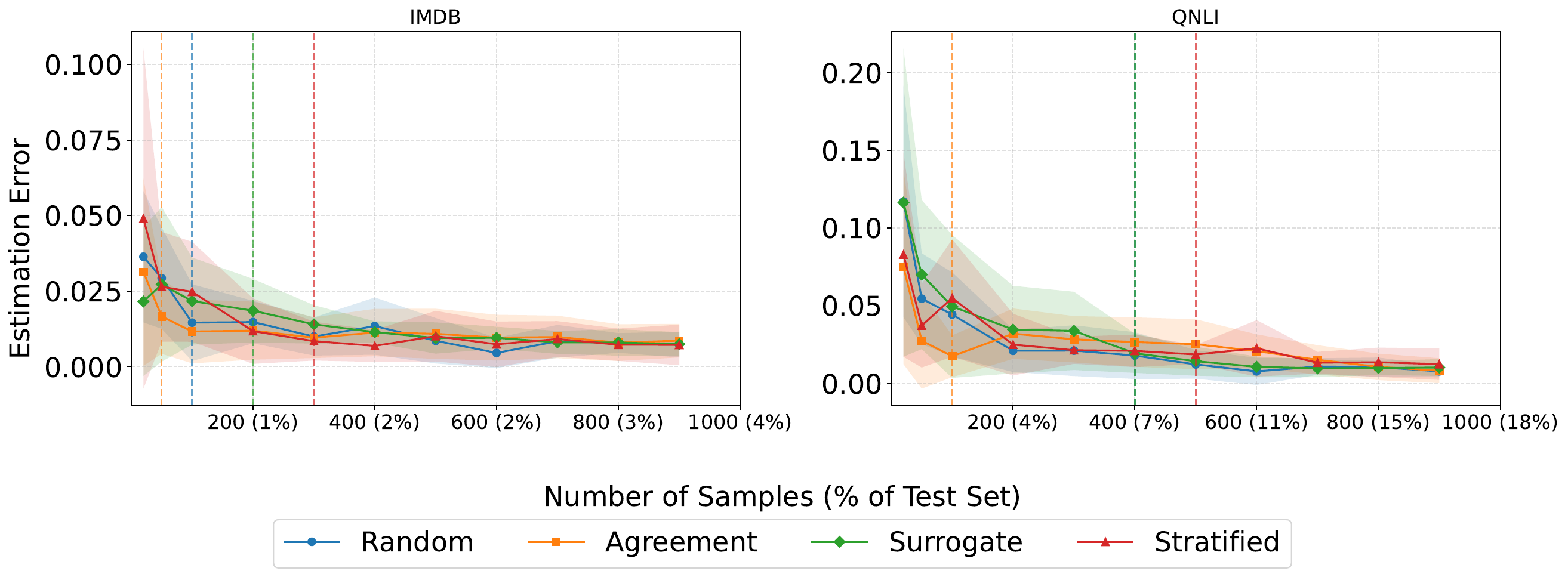}
    \caption{Results using the stopping criterion approach in order to avoid annotating all the samples ($\downarrow$ is better). Vertical lines represents when the algorithms stop.}
    \label{fig:stopping_criterion}
\end{figure}

\subsection*{RQ1: \at}
\Cref{fig:all_datasets} shows the accuracy estimation error of the top 4 strategies. We report results on six datasets spanning two classification, two summarization, one NER, and one POS tagging task. For the summarization datasets we report ROUGE-1 estimation error instead of accuracy.
Overall, all strategies perform comparably, with relatively small differences across methods. Among them, the proposed Agreement strategy tends to achieve lower estimation error, particularly at reduced values of budget $B$, where it also manages to outperform the Random baseline. Furthermore, across most datasets, the estimation error remains low even under tight budget constraints, highlighting the practical effectiveness of \at. POS tagging results follows a similar trend, with Agreement achieving competitive performance. On the UD-EWT dataset, Agreement achieves 95\% annotation reduction with an estimation error below 1\%. For NER, the curves display more irregular patterns, likely due to class imbalance, yet Agreement remains consistently effective, confirming its robustness across different task types. Full results (also in terms of Precision and Recall) are available in the repository. An additional plot including all the \at strategies is provided in \cref{fig:additional}.

\paragraph{Stopping Mechanism}

\begin{table}[t]
\centering
\caption{Time (in seconds) and annotation cost (in USD, assuming \$0.02 per sample as from \href{https://labelyourdata.com/pricing}{LabelYourData}) for different \at strategies ($\downarrow$ is better). \textit{S}: stopping criterion applied on the strategy; \textit{Full}: entire dataset annotated; \textit{NA}: stopping condition not met.}
\label{app:table_active_testing_cost}
\resizebox{\columnwidth}{!}{%
\begin{tabular}{c|cccc|cccc}
\toprule
& \multicolumn{4}{c}{IMDB} & \multicolumn{4}{c}{QNLI}\\
$B$  & T\textsubscript{Agreement}  & T\textsubscript{Random} &  T\textsubscript{Surrogate} & Cost & T\textsubscript{Agreement}  & T\textsubscript{Random} &  T\textsubscript{Surrogate} & Cost \\
\midrule
100 & 5.51 & 0.01 & 12.47 & 2 & 4.97 & 0.01 & 14.31 & 2\\
500 & 5.51 & 0.001 & 12.52 & 10 & 5.05  & 0.01 & 14.32 & 10 \\
900 & 5.52 & 0.01 & 59.05 & 18 & 5.08 & 0.01 & 43.06 & 18 \\
S-Agreement & 5.49 & - & - & \textbf{1} & 5.05 & - & - & \textbf{2} \\
S-Random & - & 0.01 & - & 2 & - & NA & - & NA \\
S-Surrogate & - & - & 12.74 & 4 & - & - & 14.71 & 8 \\
Full & 137.96 & 0.01 & 1,476.36 & 110 & 27.43 & 0.01 & 232.52 & 500 \\
\bottomrule
\end{tabular}}
\end{table} 

\Cref{fig:stopping_criterion} shows the results of the proposed stopping criterion approach on the QNLI and IMDB datasets with $B=1000$ and $\tau=0.01$. On IMDB, the Agreement approach stops the annotation process after 50 samples, meaning that it reaches values of estimation error lower than $\tau$. A similar behavior is observed for Stratified Random, which requires 250 samples, while Surrogate needs 200 samples. Comparable results are obtained on QNLI, where Agreement and Surrogate stop after 100 and 400 samples, respectively.
This means that, with respect to the full budget $B=1000$, our stopping criterion based on Agreement avoids annotating 900 samples on QNLI and 950 on IMDB.
\Cref{app:table_active_testing_cost} reports computational time and annotation cost for each strategy on IMDB and QNLI. Applying \at reduces annotation costs compared to evaluating the full test set: for instance, on QNLI the cost drops from \$500 to \$10 at $B=500$, a $50\times$ saving. Notably, Agreement combined with the stopping criterion consistently yields the lowest annotation cost across both datasets, confirming that the two contributions complement each other effectively. The cost of computing embeddings is negligible, requiring less than one hour. For clarity and due to space constraints, we report results on two datasets with additional plots in our repository.

\paragraph{Agreement strategy}
We perform an ablation to evaluate Agreement with different numbers of attention heads (4, 8, and 16). The main results use 8 heads. The results are highly consistent, indicating that the method is not sensitive to this hyperparameter. \Cref{tab:agreement_heads} reports the estimation error in terms of accuracy obtained with different numbers of attention heads across multiple annotation budgets.

\begin{table}[t]
    \centering
    \caption{Estimation error in terms of accuracy for Agreement with different numbers of attention heads across annotation budgets. Best results are shown in bold.}
    \label{tab:agreement_heads}
    \resizebox{\linewidth}{!}{
    \begin{tabular}{cccc}
        \toprule
        \textbf{Budget} & \textbf{4 Heads} & \textbf{8 Heads} & \textbf{16 Heads} \\
        \midrule
        100 & $0.0084 \pm 0.0058$ & $\mathbf{0.0075 \pm 0.0062}$ & $0.0077 \pm 0.0066$ \\
        500 & $0.0047 \pm 0.0037$ & $\mathbf{0.0038 \pm 0.0034}$ & $0.0039 \pm 0.0034$ \\
        900 & $0.0039 \pm 0.0021$ & $\mathbf{0.0034 \pm 0.0025}$ & $0.0034 \pm 0.0026$ \\
        \bottomrule
    \end{tabular}}
\end{table}

We further replace the randomly initialized attention module with the attention maps from the last layer of the pretrained BERT encoder in~\cref{tab:agreement_attention}. The performance (estimation error in terms of accuracy) remains comparable overall, showing that Agreement does not rely on a learned attention mechanism.

\begin{table}[t]
    \centering
    \caption{Estimation error in terms of accuracy for Agreement using randomly initialized attention and attention maps from the last layer of a pretrained BERT encoder. Best results are shown in bold.}
    \label{tab:agreement_attention}
    \resizebox{\linewidth}{!}{
    \begin{tabular}{ccc}
        \toprule
        \textbf{Budget} & \textbf{Random Attention} & \textbf{BERT Attention} \\
        \midrule
        100 & $\mathbf{0.0075 \pm 0.0012}$ & $0.0146 \pm 0.0015$ \\
        500 & $\mathbf{0.0038 \pm 0.0004}$ & $0.0038 \pm 0.0003$ \\
        900 & $\mathbf{0.0034 \pm 0.0005}$ & $0.0038 \pm 0.0006$ \\
        \bottomrule
    \end{tabular}}
\end{table}

Overall, these experiments indicate that Agreement is robust to the attention configuration and that its performance is driven by the quality of the contextual representations rather than by a specific attention parameterization.

\subsection*{RQ2: \at with Cost-Based Priors}
\cref{fig:multilingual} shows the results of the experiments on the Multilingual dataset. For the mixed strategy, which leverages the cross-lingual capabilities of our approach, two values of prior are considered: $p=0.9$, which favors selecting English samples (cheaper), and $p=0.6$, which provides a more balanced selection across languages. 
Allowing language selection based on a prior benefits Surrogate and Agreement, achieving estimation errors lower than single-language sampling. For Random, high priors (\eg, $p=0.9$) introduce bias toward one language, negatively affecting accuracy. Nevertheless, setting $p=0.6$ reduces annotation costs by favoring cheaper English samples, while maintaining performance within 1\% of the balanced setup, demonstrating that cost-aware \at enables multilingual evaluation with minimal accuracy loss.
Unlike results observed in previous experiments, in this setting the \at strategies consistently outperform random sampling.

\begin{table}[!t]
\centering
\caption{Per-class F1 on minority class detection ($\uparrow$ is better). Agreement and Surrogate outperform Random. Top: Emotion, bottom: PubMed.}
\label{tab:minority_detection}
\begin{minipage}{\linewidth}
\centering
\resizebox{\linewidth}{!}{%
\begin{tabular}{cccc}
\toprule
$B$ & Agreement & Random & Surrogate-SVM \\
\midrule
75 & \textbf{0.031 $\pm$ 0.006} & 0.022 $\pm$ 0.012 & \underline{0.022 $\pm$ 0.005} \\
150 & \textbf{0.042 $\pm$ 0.008} & \underline{0.036 $\pm$ 0.009} & 0.029 $\pm$ 0.006 \\
450 & \textbf{0.044 $\pm$ 0.006} & 0.033 $\pm$ 0.007 & \underline{0.035 $\pm$ 0.008} \\
750 & \textbf{0.047 $\pm$ 0.003} & 0.043 $\pm$ 0.003 & \underline{0.044 $\pm$ 0.006} \\
\bottomrule
\end{tabular}}
\end{minipage} \\
\begin{minipage}{\linewidth}
\centering
\resizebox{\linewidth}{!}{
\begin{tabular}{cccc}
\toprule
$B$ & Agreement & Random & Surrogate-SVM \\
\midrule
100 & \textbf{0.096 $\pm$ 0.001} & 0.041 $\pm$ 0.008 &  \underline{0.072 $\pm$ 0.003} \\
500 & \textbf{0.151 $\pm$ 0.003} & 0.074 $\pm$ 0.007 & \underline{0.103 $\pm$ 0.004} \\
1500 & \textbf{0.145 $\pm$ 0.006} & 0.094 $\pm$ 0.007 & \underline{0.133 $\pm$ 0.007}  \\
3000 & \textbf{0.125 $\pm$ 0.005} & 0.102 $\pm$ 0.006  & \underline{0.113 $\pm$ 0.005} \\
\bottomrule
\end{tabular}}
\end{minipage}
\end{table}

\subsection*{RQ3: Minority Class Sample Detection}

\Cref{tab:minority_detection} presents the results on minority class detection on two datasets and three \at strategies. We focus on Emotion and PubMed as they exhibit the highest class imbalance (see \cref{app:table_datasets}). To evaluate how effectively each strategy identifies minority-class samples, we treat minority detection as a binary retrieval task: we compute Precision as the fraction of selected samples that belong to the minority class, and F1 accordingly. Across both datasets, Agreement consistently achieves the highest F1, followed by Surrogate, with Random performing worst. The advantage of \at strategies is particularly pronounced at lower budgets: on PubMed at $B=100$, Agreement achieves an F1 of $0.096$, more than twice the Random baseline ($0.041$), demonstrating that informed sampling is especially beneficial when annotation resources are scarce. These results confirm that, as in the multilingual setting, \at strategies prove advantageous when targeting specific subsets of interest, with Agreement being particularly effective at detecting underrepresented classes.
We report results on two datasets; results on the remaining data are consistent and can be found in our repository.

\section{Conclusions and Future Work}

We presented a framework for \at in NLP, demonstrating its effectiveness in reducing annotation costs while maintaining reliable model evaluation. The adaptive approach successfully addresses the challenge of determining optimal sample sizes, making \at practical for real-world applications.

\paragraph{Limitations}
Our work focuses on static model evaluation settings, assuming fixed predictors and test distributions. Extending \at to evolving models, continuously updated benchmarks, or interactive NLP systems remains an open direction for future research.

\section*{Acknowledgments}
Maria Sofia Bucarelli acknowledges the French government National Research Agency (ANR) through the UCA JEDI (ANR-15-IDEX-01),  through the EUR DS4H (ANR-17-EURE-004), and through the 3IA Cote d'Azur Investments in the project with the reference number ANR-23-IACL-0001. Fabrizio Silvestri and Antonio Purificato acknowledge the project SAP\_RICERCA\_2025\_EVOLVE\_SILVES\_F\_01.

\bibliography{biblio}

@inproceedings{nguyen2018activetestingefficientrobust,
  title={Active testing: An efficient and robust framework for estimating accuracy},
  author={Nguyen, Phuc and Ramanan, Deva and Fowlkes, Charless},
  booktitle={International Conference on Machine Learning},
  pages={3759--3768},
  year={2018},
  organization={PMLR}
}

@inproceedings{kossen2021active,
  title={Active testing: Sample-efficient model evaluation},
  author={Kossen, Jannik and Farquhar, Sebastian and Gal, Yarin and Rainforth, Tom},
  booktitle={International Conference on Machine Learning},
  pages={5753--5763},
  year={2021},
  organization={PMLR}
}

@inproceedings{bias,
  title={On Statistical Bias In Active Learning: How and When to Fix It},
  author={Farquhar, Sebastian and Gal, Yarin and Rainforth, Tom},
  booktitle={International Conference on Learning Representations},
  year = {2021}
}

@inproceedings{li2024active,
  title={Active Learning for Abstractive Text Summarization via LLM-Determined Curriculum and Certainty Gain Maximization},
  author={Li, Dongyuan and Zhang, Ying and Wang, Zhen and Tan, Shiyin and Kosugi, Satoshi and Okumura, Manabu},
  booktitle={Findings of the Association for Computational Linguistics: EMNLP 2024},
  pages={8959--8971},
  year={2024}
}

@misc{nofree,
      title={No Free Lunch in Active Learning: LLM Embedding Quality Dictates Query Strategy Success}, 
      author={Lukas Rauch and Moritz Wirth and Denis Huseljic and Marek Herde and Bernhard Sick and Matthias Aßenmacher},
      year={2025},
      eprint={2506.01992},
      archivePrefix={arXiv},
      primaryClass={cs.CL},
      url={https://arxiv.org/abs/2506.01992}, 
}

@inproceedings{activeglae,
author = {Rauch, Lukas and A\ss{}enmacher, Matthias and Huseljic, Denis and Wirth, Moritz and Bischl, Bernd and Sick, Bernhard},
title = {ActiveGLAE: A Benchmark for Deep Active Learning with Transformers},
year = {2023},
isbn = {978-3-031-43411-2},
publisher = {Springer-Verlag},
address = {Berlin, Heidelberg},
url = {https://doi.org/10.1007/978-3-031-43412-9_4},
doi = {10.1007/978-3-031-43412-9_4},
booktitle = {Machine Learning and Knowledge Discovery in Databases: Research Track: European Conference, ECML PKDD 2023, Turin, Italy, September 18–22, 2023, Proceedings, Part I},
pages = {55–74},
numpages = {20},
location = {Turin, Italy}
}

@inproceedings{liu2023zero,
  title={Zero-Shot Text Classification via Self-Supervised Tuning},
  author={Liu, Chaoqun and Zhang, Wenxuan and Chen, Guizhen and Wu, Xiaobao and Tuan, Luu Anh and Chang, Chip Hong and Bing, Lidong},
  booktitle={Findings of the Association for Computational Linguistics: ACL 2023},
  pages={1743--1761},
  year={2023}
}

@inproceedings{parekh-etal-2025-dynamic,
    title = "Dynamic Strategy Planning for Efficient Question Answering with Large Language Models",
    author = "Parekh, Tanmay  and
      Prakash, Pradyot  and
      Radovic, Alexander  and
      Shekher, Akshay  and
      Savenkov, Denis",
    editor = "Chiruzzo, Luis  and
      Ritter, Alan  and
      Wang, Lu",
    booktitle = "Findings of the Association for Computational Linguistics: NAACL 2025",
    month = apr,
    year = "2025",
    address = "Albuquerque, New Mexico",
    publisher = "Association for Computational Linguistics",
    url = "https://aclanthology.org/2025.findings-naacl.336/",
    doi = "10.18653/v1/2025.findings-naacl.336",
    pages = "6038--6059",
    ISBN = "979-8-89176-195-7",
}

@inproceedings{wei2018clinical,
  title={Clinical text annotation--what factors are associated with the cost of time?},
  author={Wei, Qiang and Franklin, Amy and Cohen, Trevor and Xu, Hua},
  booktitle={AMIA Annual Symposium Proceedings},
  volume={2018},
  pages={1552},
  year={2018}
}

@article{karim2025transforming,
  title={Transforming Data Annotation with AI Agents: A Review of Architectures, Reasoning, Applications, and Impact},
  author={Karim, Md Monjurul and Khan, Sangeen and Van, Dong Hoang and Liu, Xinyue and Wang, Chunhui and Qu, Qiang},
  journal={Future Internet},
  volume={17},
  number={8},
  pages={353},
  year={2025},
  publisher={MDPI}
}

@inproceedings{ahmadnia-etal-2025-active,
    title = "Active Few-Shot Learning for Text Classification",
    author = "Ahmadnia, Saeed  and
      Yousefi Jordehi, Arash  and
      Hosseini Khasheh Heyran, Mahsa  and
      Mirroshandel, Seyed Abolghasem  and
      Rambow, Owen  and
      Caragea, Cornelia",
    editor = "Chiruzzo, Luis  and
      Ritter, Alan  and
      Wang, Lu",
    booktitle = "Proceedings of the 2025 Conference of the Nations of the Americas Chapter of the Association for Computational Linguistics: Human Language Technologies (Volume 1: Long Papers)",
    month = apr,
    year = "2025",
    address = "Albuquerque, New Mexico",
    publisher = "Association for Computational Linguistics",
    url = "https://aclanthology.org/2025.naacl-long.340/",
    doi = "10.18653/v1/2025.naacl-long.340",
    pages = "6677--6694",
    ISBN = "979-8-89176-189-6",
}

@article{kirsch2019batchbald,
  title={Batchbald: Efficient and diverse batch acquisition for deep bayesian active learning},
  author={Kirsch, Andreas and Van Amersfoort, Joost and Gal, Yarin},
  journal={Advances in neural information processing systems},
  volume={32},
  year={2019}
}

@inproceedings{imberg2020optimal,
  title={Optimal sampling in unbiased active learning},
  author={Imberg, Henrik and Jonasson, Johan and Axelson-Fisk, Marina},
  booktitle={International Conference on Artificial Intelligence and Statistics},
  pages={559--569},
  year={2020},
  organization={PMLR}
}

@inproceedings{maharanamathbb,
  title={$\mathbb{D}^2$Pruning: Message Passing for Balancing Diversity \& Difficulty in Data Pruning},
  author={Maharana, Adyasha and Yadav, Prateek and Bansal, Mohit},
  booktitle={The Twelfth International Conference on Learning Representations},
  year={2024}
}

@article{kreutzer2022quality,
  title={Quality at a glance: An audit of web-crawled multilingual datasets},
  author={Kreutzer, Julia and Caswell, Isaac and Wang, Lisa and Wahab, Ahsan and Van Esch, Daan and Ulzii-Orshikh, Nasanbayar and Tapo, Allahsera and Subramani, Nishant and Sokolov, Artem and Sikasote, Claytone and others},
  journal={Transactions of the Association for Computational Linguistics},
  volume={10},
  pages={50--72},
  year={2022},
  publisher={MIT Press One Rogers Street, Cambridge, MA 02142-1209, USA journals-info~…}
}

@inproceedings{gururangan-etal-2018-annotation,
    title = "Annotation Artifacts in Natural Language Inference Data",
    author = "Gururangan, Suchin  and
      Swayamdipta, Swabha  and
      Levy, Omer  and
      Schwartz, Roy  and
      Bowman, Samuel  and
      Smith, Noah A.",
    editor = "Walker, Marilyn  and
      Ji, Heng  and
      Stent, Amanda",
    booktitle = "Proceedings of the 2018 Conference of the North {A}merican Chapter of the Association for Computational Linguistics: Human Language Technologies, Volume 2 (Short Papers)",
    month = jun,
    year = "2018",
    address = "New Orleans, Louisiana",
    publisher = "Association for Computational Linguistics",
    url = "https://aclanthology.org/N18-2017/",
    doi = "10.18653/v1/N18-2017",
    pages = "107--112",
}

@article{zhao2024adamergex,
  title={Adamergex: Cross-lingual transfer with large language models via adaptive adapter merging},
  author={Zhao, Yiran and Zhang, Wenxuan and Wang, Huiming and Kawaguchi, Kenji and Bing, Lidong},
  journal={arXiv preprint arXiv:2402.18913},
  year={2024}
}

@inproceedings{
zhang2025how,
title={How Benchmark Prediction from Fewer Data Misses the Mark},
author={Guanhua Zhang and Florian E. Dorner and Moritz Hardt},
booktitle={NeurIPS 2025 Workshop on Evaluating the Evolving LLM Lifecycle: Benchmarks, Emergent Abilities, and Scaling},
year={2025},
url={https://openreview.net/forum?id=3pUtKEctwR}
}

@inproceedings{pmlr-v267-li25bp,
  title = 	 {Active Evaluation Acquisition for Efficient {LLM} Benchmarking},
  author =       {Li, Yang and Ma, Jie and Ballesteros, Miguel and Benajiba, Yassine and Horwood, Graham},
  booktitle = 	 {Proceedings of the 42nd International Conference on Machine Learning},
  pages = 	 {35581--35602},
  year = 	 {2025},
  editor = 	 {Singh, Aarti and Fazel, Maryam and Hsu, Daniel and Lacoste-Julien, Simon and Berkenkamp, Felix and Maharaj, Tegan and Wagstaff, Kiri and Zhu, Jerry},
  volume = 	 {267},
  series = 	 {Proceedings of Machine Learning Research},
  month = 	 {13--19 Jul},
  publisher =    {PMLR},
  url = 	 {https://proceedings.mlr.press/v267/li25bp.html},
}

@inproceedings{lin2004rouge,
  title={Rouge: A package for automatic evaluation of summaries},
  author={Lin, Chin-Yew},
  booktitle={Text summarization branches out},
  pages={74--81},
  year={2004}
}

@article{horvitz1952generalization,
  title={A generalization of sampling without replacement from a finite universe},
  author={Horvitz, Daniel G and Thompson, Donovan J},
  journal={Journal of the American statistical Association},
  volume={47},
  number={260},
  pages={663--685},
  year={1952},
  publisher={Taylor \& Francis}
}

@inproceedings{zhen2025enhancing,
  title={Enhancing LLM-as-a-Judge through Active-Sampling-based Prompt Optimization},
  author={Zhen, Cheng and Zheng, Ervine and Kuang, Jilong and Tso, Geoffrey Jay},
  booktitle={Proceedings of the 63rd Annual Meeting of the Association for Computational Linguistics (Volume 6: Industry Track)},
  pages={960--970},
  year={2025}
}

@article{zouhar2025select,
  title={How to Select Datapoints for Efficient Human Evaluation of NLG Models?},
  author={Zouhar, Vil{\'e}m and Cui, Peng and Sachan, Mrinmaya},
  journal={arXiv preprint arXiv:2501.18251},
  year={2025}
}

@article{houlsby2011bayesian,
  title={Bayesian active learning for classification and preference learning},
  author={Houlsby, Neil and Husz{\'a}r, Ferenc and Ghahramani, Zoubin and Lengyel, M{\'a}t{\'e}},
  journal={arXiv preprint arXiv:1112.5745},
  year={2011}
}

\appendix

\section{Theoretical Results}

\subsection{Unbiased Estimators}
\label{app:unbiased_est}

In this section, we derive the unbiased estimators used in our experiments for Accuracy, Precision and Recall. 
All quantities are computed under importance sampling, where each selected sample $i$ is drawn according to a probability $q_i$, and is assigned an importance weight $w_i = \frac{1}{N q_i}$.

\paragraph{Unbiased Accuracy}
The unbiased estimator of accuracy is computed as:
\begin{equation*}
     \widehat{A} = \frac{1}{B} \sum_{i=1}^{B} \frac{A(f(x_i),y_i)}{Nq_i}
\end{equation*}

\paragraph{Unbiased Precision and Recall}

For each class $c\in\{1,\dots,C\}$, we define the following importance-weighted true positives as:

\begin{equation*}
    \widehat{TP_c} = \sum_{i=1}^{M} w_i \, \mathds{1}[\hat{y}_i = c] \, \mathds{1}[y_i = c]
\end{equation*}

Since the expectation of a ratio differs from the ratio of expectations, we cannot apply importance weighting to the denominator. Instead, we use the predicted instances from the full test set:
\begin{equation*}
    \mathbb{E}\left[ \frac{U}{V}\right] \neq \frac{\mathbb{E}\left[U\right]}{\mathbb{E}\left[V\right]} \Rightarrow \mathbb{E}\left[ \frac{\widehat{TP}_c}{\mathrm{PI}_c} \right] \neq \frac{\mathbb{E}\left[ \widehat{TP}_c \right]}{\mathbb{E}\left[ PI_c \right]}
\end{equation*}

The predicted instances can be defined as:

\begin{equation*}
    \mathrm{PI}_c = \sum_{i=1}^{N} \mathds{1}[\hat{y}_i = c]
\end{equation*}

Same mechanism applies for the true instances, which are defined as:
\begin{equation*}
    \mathrm{TI}_c = \sum_{i=1}^{N} \mathds{1}[y_i = c]
\end{equation*}

The unbiased Precision $(\hat{P})$ and Recall $(\hat{R})$ are respectively defined as:

\begin{align*}
    \widehat{\mathrm{P}}
    &= \frac{1}{C} \sum_{c=1}^{C}\frac{\widehat{TP}_c}{\mathrm{PI}_c} \\    \widehat{\mathrm{R}}
    &= \frac{1}{C} \sum_{c=1}^{C}\frac{\widehat{TP}_c}{\mathrm{TI}_c} \\
\end{align*}

\paragraph{Unbiased ROUGE}

Given a predicted summary $f(x_i)$ and a reference summary $y_i$, the ROUGE-N score for a single example is defined as:
\begin{align*}
    & \text{ROUGE-N}(f(x_i), y_i) = \\
    &= \frac{\sum_{g \in \mathcal{G}_N(y_i)} \min\!\big(C(g, f(x_i)),\; C(g, y_i)\big)}{\sum_{g \in \mathcal{G}_N(y_i)} C(g, y_i)}
\end{align*}

where $\mathcal{G}_N(y_i)$ is the set of N-grams in the reference $y_i$ and $C(g, s)$ denotes the count of N-gram $g$ in sequence $s$ \cite{lin2004rouge}.

Since ROUGE-N can be decomposed as a per-example average:
\begin{equation*}
    \text{ROUGE-N} = \frac{1}{N} \sum_{i=1}^{N} \text{ROUGE-N}(f(x_i), y_i)
\end{equation*}

we can apply the Inverse Probability Weighted Estimator to obtain 
an unbiased estimate:
\begin{equation*}
    \widehat{\text{ROUGE-N}} = \frac{1}{B} \sum_{i=1}^{B} 
    \frac{\text{ROUGE-N}(f(x_i), y_i)}{N q_i}
\end{equation*}

where $q_i = q(x_i; X_{1:i-1}, X)$ is the probability mass 
for datum $x_i$ of being the next to be sampled, depending on the 
active testing strategy. This holds for the macro-averaged formulation 
of ROUGE, where the final score is obtained as the mean of 
per-example ROUGE scores.

\paragraph{Remark.} Note that while Precision and Recall admit unbiased estimators (by fixing the denominator to the full test set counts), the F1 score does not. As a nonlinear function of $P$ and $R$, unbiasedness does not transfer, and thus F1 cannot be estimated without bias under importance sampling.

\subsection{Proof of Unbiasedness and Convergence}
\label{app:unbiasedness}

To prove unbiasedness, we need to show that $\expected[\widehat{M}] = M$.

Starting with the expectation:
\begin{equation*}
    \expected[\widehat{M}] = \expected\left[\frac{1}{B} \sum_{i=1}^B \frac{M_i}{Nq}\right]
\end{equation*}

By linearity of expectation:
\begin{equation*}
    \expected[\widehat{M}] = \frac{1}{B} \sum_{i=1}^B \expected\left[\frac{M_i}{Nq_i}\right]
\end{equation*}

For each term in the sum:
\begin{equation*}
    \expected\left[\frac{M_i}{Nq}\right] = \sum_{M_i} \frac{M_i}{Nq} \mathbb{P}(M_i)
\end{equation*}

Since $q$ is the probability of selecting the index, when we multiply $\frac{M_i}{Nq}$ by $q$ (the selection probability), we get:
\begin{equation*}
    \expected\left[\frac{M_i}{Nq}\right] = \frac{M}{N}
\end{equation*}

Therefore:
\begin{equation*}
    \expected[\widehat{M}] = \frac{1}{B} \sum_{i=1}^B \frac{M_i}{N} = M
\end{equation*}

Regarding the convergence in expectation, by the Law of Large Numbers, since each term $\frac{M_i}{Nq}$ is independent and identically distributed with finite expectation $M$, as $B \to \infty$:
\begin{equation*}
    \widehat{M} \xrightarrow{p} M
\end{equation*}
The variance of the estimator decreases as $1/B$, showing that it becomes more precise with more samples:

\begin{align*}
    &Var(\widehat{M}) = Var\left(\frac{1}{B} \sum_{i=1}^B \frac{M_i}{Nq}\right) = \\
    &= \frac{1}{B^2} \sum_{i=1}^B Var\left(\frac{M_i}{Nq}\right)
\end{align*}

Since the terms are independent, and assuming finite variance $\sigma^2$:

\begin{equation*}
    Var(\widehat{M}) = \frac{\sigma^2}{B}
\end{equation*}

\subsection{Limits of the estimators of \citet{bias}}
\label{app:limits_estimator}

\citet{bias} define two estimators. We start analyzing the first one. To simplify the computations we assume that all the samples have an accuracy equal to $1$, obtaining an accuracy of $1$:
\begin{align*}
     A^{\text{PURE}} &= \frac{1}{M} \sum_{m=1}^M\left( w_m + \frac{M-m}{N}\right)\text{Acc}_{m},\ \\
   A^{\text{PURE}} &= \frac{1}{M} \sum_{m=1}^M\left( 1 + \frac{M-m}{N}\right) 1,\ \\ 
     A^{\text{PURE}} &= \frac{1}{M} \sum_{m=1}^M 1 + \frac{1}{M}\sum_{m=1}^M \frac{M}{N} - \frac{1}{M}\sum_{m=1}^M  \frac{m}{N}\\
      A^{\text{PURE}} &= \frac{1}{M} M + \frac{1}{M} \frac{M^2}{N} - \frac{1}{M}  \frac{M(M+1)}{2N}\\
      A^{\text{PURE}} &= 1 + 
    \frac{M}{N} - \frac{M+1}{2N} \\   
      A^{\text{PURE}} &= 1 + 
    \frac{2M-M-1}{N}\\
       A^{\text{PURE}} &= 1 + 
    \frac{M-1}{N}
\end{align*}
This means that the PURE, while working well in the case of the loss function,  overestimates the value of the accuracy, easily reaching values which are bigger than $1$.
For example, if we assume a total number of samples $N=500$ and a budget $B=450$, we obtain an accuracy estimated $A^{\text{PURE}}=1.898$.

Now we will focus our analysis on the second estimator proposed by \citet{bias}, with the same assumption of all the accuracy values being equal to $1$:

\begin{align*}
    A^{\text{LURE}} &= \frac{1}{M} \sum_{m=1}^M v_m\text{Acc}_m\\
        A^{\text{LURE}} &= \frac{1}{M} \sum_{m=1}^M 1 + \frac{N-M}{N-m}\Biggl( -1 + \\
        & + \frac{1}{(N-m+1)q}\Biggr) \\
        A^{\text{LURE}} &= \frac{1}{M} \sum_{m=1}^M 1 + \frac{N-M}{N-m}\left(\frac{1}{\frac{N-m+1}{N} } -1 \right)\\
        A^{\text{LURE}} &= 1 + \frac{1}{M} \sum_{m=1}^M
        \frac{N-M}{N-m}\left(\frac{1}{
    \frac{N-m+1}{N}}-1 \right)\\
          A^{\text{LURE}} &= 1 + \frac{1}{M} \sum_{m=2}^M
        \frac{N-M}{N-m}\left(\frac{N}{N-m+1} -1 \right)\\
      {A^{\text{LURE}}}&=  1 + \frac{N-M}{M} \sum_{m=2}^M
        \frac{1}{N-m}\left(\frac{m-1}{N-m+1 } \right)\\
        A^{\text{LURE}} &=  1 + \frac{N-M}{M} \sum_{k=N-M}^{N-2} \frac{N-1-k}{k(k+1)}\\
     A^{\text{LURE}} &=  1 + \frac{N-M}{M}  \left[ \sum_{k=N-M}^{N-2} \frac{N-1}{k} - \frac{N}{k+1} \right]\\
        A^{\text{LURE}} &=  1 + \frac{N-M}{M} \Biggl[ (N-1)\sum_{k=N-M}^{N-2} \frac{1}{k}  \\
        & \quad - N\sum_{k=N-M}^{N-2}\frac{1}{k+1} \Biggl]\\
           A^{\text{LURE}} &=  1 + \frac{N-M}{M} \Biggl[  (N-1)\sum_{k=N-M}^{N-2} \frac{1}{k} \\
       & \quad -  N\sum_{k=N-M+1}^{N-1}\frac{1}{k} \Biggl]\\
\end{align*}

Denoting by $H_{N} = \sum_{h=1}^N \frac{1}{h}$, 
we can rewrite: $$ \sum_{k=N-M}^{N-2} \frac{1}{k} = H_{N-2} - H_{N-M}  $$
And: $$ \sum_{k=N-M+1}^{N-1}\frac{1}{k} = H_{N-1} - H_{N-M+1}.$$

Now:

\begin{align*}
   A^{\text{LURE}} &=  1 + \frac{N-M}{M} \Biggl[  (N-1)(H_{N-2}  - H_{N-M})  \\
   & \quad - N (H_{N-1} - H_{N-M+1})  \Biggl]   \\
    A^{\text{LURE}} &= 1 + \frac{N-M}{M} \Biggl[  N (H_{N-2}- H_{N-1})  \\
   & \quad + N(  H_{N-M+1}- H_{N-M})  + \\
   & \quad \quad -  H_{N-2} + H_{N-M}  \Biggl]\\
     A^{\text{LURE}} &= 1 + \frac{N-M}{M} \Biggl[ \frac{-N}{N-1}  + \frac{N}{N-M+1}\\
   & \quad +   - H_{N-2}  + H_{N-M}  \Biggl]\\
    A^{\text{LURE}} &= 1 + \frac{N}{M} \Biggl[ - \frac{N-M}{N-1}  + \frac{N-M}{N-M+1} \Biggl]\\
   & \quad + \frac{N-M}{M}  \Biggl[ - H_{N-2}  + H_{N-M}  \Biggl]\\
\end{align*}

We can see that LURE suffers of the same problem of LURE. For example, if we set $N = 500$ $B =450$, we obtain an accuracy estimated $A^{\text{LURE}}=1.724$.

\subsection{Stopping Criterion convergence}
\label{app:proof_convergence}

\begin{proof}
We denote by $m_i := M(f(x_i),y_i)$. 
Notice that the sample size is deterministically $B$ and $q_i>0$ for all $i$. $B\to N$, hence $q_i \to 1$ and $q_{ij}\to 1$ for all $i\neq j$.
 $\frac{1}{N}\sum_{i=1}^N m_i^2 \le M < \infty$ uniformly in $N$.

We denote the sample membership indicators by $I_i := \mathbf{1}\{i\in S\}$ and define:
\begin{align}\label{eq:b_i}
    b_i \;:=\; m_i\!\left(\frac{1}{B} - \frac{1}{N q_i}\right),
    \qquad\text{so that}\\ \label{eq:D}
    D \;:=\; \widehat{M}_{\mathrm{random}} - \widehat{M}_{\mathrm{HT}}
           \;=\; \sum_{i=1}^N I_i\, b_i .
\end{align}
Let $\Delta_{ij}:=q_{ij}-q_iq_j$ denote the joint inclusion covariance
(with the convention $\Delta_{ii}=q_i(1-q_i)$).
Not that $q_{ij} = \mathbb{P}(i, j \in S)$
We can see that the mean of the gap vanishes.
Taking design expectations and using $\mathbb{E}(I_i)=q_i$:
\begin{align*}
\mathbb{E}[D]
&= \frac{1}{B}\sum_{i=1}^N q_i m_i \;-\; \frac{1}{N}\sum_{i=1}^N m_i
\\
&= \sum_{i=1}^N m_i\!\left(\frac{q_i}{B}-\frac{1}{N}\right).
\end{align*}
As $B\to N$ we have $q_i\to 1$ for every $i$, hence
$\frac{q_i}{B}-\frac{1}{N} \to \frac{1}{N}-\frac{1}{N}=0$.
Using \cref{eq:D} to control the magnitudes of $a_i$, it follows that $\mathbb{E}[D]\to 0$.

Also the variance Gap vanishes. 
Because $D$ is linear in the indicators, its design variance can be written in covariance form:
\begin{equation*}
\operatorname{Var}(D)
= \sum_{i=1}^N\sum_{j=1}^N \Delta_{ij}\, b_i b_j .
\end{equation*}

\begin{equation}
\label{eq:complete}
\operatorname{Var}(D)
= \sum_{i \neq j }(q_{i} - q_i^2)\, b_i^2 +  (q_{ij} - q_i q_j)\, b_i b_j 
\end{equation}

Since the sample size is fixed. Let's say $ \sum_{k=1}^N I_k = B$ almost surely, then:
$$ 0 = \text{Cov}(I_i, \sum_{k=1}^N I_k) = \sum_{k=1}^N \text{Cov}(I_i,  I_k). $$
So:
$$ 0 = \sum_{k=1}^N \Delta_{ik}.$$
Hence $ \sum_{i\neq j} \Delta_{i,j} = - \Delta_{ii} = - q_i(1-q_i)$.  
The total covariance weight goes to 0 as $ B \to N$, indeed in this case $q_i \to 1 $ and $1- q_i \to 0$. 
% \begin{equation}\label{eq:YG}
% \operatorname{Var}(D)
% = \frac{1}{2}\sum_{i\neq j} \big(-\Delta_{ij}\big)\, (b_i - b_j)^2 .
% \end{equation}

From \cref{eq:b_i}, implies $b_i \to m_i\left(\frac{1}{N}-\frac{1}{N}\right)=0$ and
$-\Delta_{ij}\to 0$ for each pair $i\neq j$ as $B\to N$.
The sequence $\{m_i\}$ has uniformly bounded second moments. 
Therefore, each summand in \cref{eq:complete} tends to $0$ and the finite sum converges to $0$:
$\operatorname{Var}(D)\to 0$.

Combining the two results,
\begin{equation*}
\mathbb{E}\!\left[D^2\right]
= \operatorname{Var}(D) + \big(\mathbb{E}[D]\big)^2 \;\longrightarrow\; 0
\end{equation*}
so $D\to 0$ in $L^2$, which implies $D\xrightarrow{P} 0$.
\end{proof}

\begin{proposition}[Convergence rate of the estimator gap]
\label{prop:convergence_rate}
    Let $\widehat{M}_{\emph{random}}^{(n)}$ be the unweighted sample mean
    and $\widehat{M}^{(n)}$ the Horvitz--Thompson estimator over a fixed
    budget of $n$ annotated samples drawn according to inclusion probabilities
    $q_i^{(n)} \in (0,1]$ with $\sum_{i=1}^N q_i^{(n)} = n$, and with
    metric values bounded $M_i \in [a,b]$. Then:
    \begin{align*}
         \mathbb{E}\!\left[|\widehat{M}_{\emph{random}}^{(n)} - \widehat{M}^{(n)}|\right]
        \; & \leq\;
        \frac{\sqrt{N}}{n}\,\|q^{(n)}\|_2\,\sigma_M
        \;+ \\
        & + \frac{(b-a)}{q_{\min}^{(n)}\sqrt{2n}},
    \end{align*}
    Where:
    \begin{align*}
    \|q^{(n)}\|_2 & = \bigl(\sum_i q_i^{(n)^2}\bigr)^{1/2} \\
    \sigma_M^2 &= \frac{1}{N}\sum_i(M_i - M)^2 \\
    q_{\min}^{(n)} &= \min_i q_i^{(n)}
    \end{align*}
    In particular:
    \begin{enumerate}
        \item \emph{(Uniform sampling).} If $q_i^{(n)} = n/N$ for all $i$,
        then the bias is exactly zero and:
        \[
            \mathbb{E}\!\left[|\widehat{M}_{\emph{random}}^{(n)} -
            \widehat{M}^{(n)}|\right] \;\leq\;
            \frac{(b-a)\,N}{\sqrt{2}\,n^{3/2}}
            \;=\; \mathcal{O}(n^{-3/2}).
        \]
        \item \emph{(Uncertainty sampling).} If $q_i^{(n)} = n u_i / U$
        with $U = \sum_j u_j$, then:
        \begin{align*}
            \mathbb{E}\!\left[|\widehat{M}_{\emph{random}}^{(n)} -
            \widehat{M}^{(n)}|\right] \; &\leq\;
            \frac{\sqrt{N}\,\|u\|_2}{U}\,\sigma_M
            \; + \\ +\;
            \frac{(b-a)\,U}{\sqrt{2}\,u_{\min}\,n^{3/2}},
        \end{align*}
        where the first term is an irreducible bias independent of $n$
        in the regime $n \ll N$.
    \end{enumerate}
\end{proposition}

\begin{proof}
By the triangle inequality:
\begin{align}
    \label{eq:decomp}
    \nonumber |\widehat{M}_{\emph{random}}^{(n)} - \widehat{M}^{(n)}|
    \; &\leq\;
    |\widehat{M}_{\emph{random}}^{(n)} - M|
    \;+ \\
    \;
    + |\widehat{M}^{(n)} - M|
\end{align}

\paragraph{Bounding $|\widehat{M}_{\emph{random}}^{(n)} - M|$.}
Define $Z_i = \mathbf{1}[i \in S_n]$ so that $\mathbb{E}[Z_i] = q_i^{(n)}$.
Since the budget is fixed, $\sum_{i=1}^N q_i^{(n)} = n$, and:
\begin{align*}
    & \mathbb{E}[\widehat{M}_{\emph{random}}^{(n)}] - M
    = \frac{1}{n}\sum_{i=1}^N q_i^{(n)} M_i - \frac{1}{N}\sum_{i=1}^N M_i \\
    &= \frac{1}{n}\sum_{i=1}^N q_i^{(n)}(M_i - M)
    + M\underbrace{\!\left(\frac{1}{n}\sum_{i=1}^N q_i^{(n)} - 1\right)}_{=\,0}
\end{align*}
Where the last term vanishes by the fixed-budget constraint. Hence:
\[
    \mathbb{E}[\widehat{M}_{\emph{random}}^{(n)}] - M
    = \frac{1}{n}\sum_{i=1}^N q_i^{(n)}(M_i - M).
\]
Applying Cauchy--Schwarz:
\begin{align}
    \label{eq:bias}
    \nonumber & \left|\mathbb{E}[\widehat{M}_{\emph{random}}^{(n)}] - M\right|
    \leq \\
    \nonumber & \leq \frac{1}{n}\left(\sum_{i=1}^N q_i^{(n)2}\right)^{\frac{1}{2}}
    \!\! \left(\sum_{i=1}^N (M_i-M)^2\right)^{\frac{1}{2}}
    \!\! = \\ 
    & = \!\!
    \frac{\sqrt{N}}{n}\,\|q^{(n)}\|_2\,\sigma_M.
\end{align}
Since the bias is deterministic,
$\mathbb{E}[|\widehat{M}_{\emph{random}}^{(n)} - M|] \leq
|\mathbb{E}[\widehat{M}_{\emph{random}}^{(n)}] - M|$,
so \cref{eq:bias} bounds the expected absolute deviation directly.

\paragraph{Bounding $|\widehat{M}^{(n)} - M|$.}
Since $\widehat{M}^{(n)}$ is unbiased for $M$ and each term
$M_i/q_i^{(n)}$ lies in $[a/q_i^{(n)},\, b/q_i^{(n)}]$,
applying Hoeffding's inequality gives:
\begin{align*}
    & \Pr(|\widehat{M}^{(n)} - M| \geq u)
    \;\leq\ \\
    & \leq 2\exp\!\left(-\frac{2n^2 u^2}{\sum_{i \in S_n}(b-a)^2/q_i^{(n)2}}\right)
    \;\leq \\ 
    & \leq 2\exp\!\left(-\frac{2n\,q_{\min}^{(n)2}\,u^2}{(b-a)^2}\right).
\end{align*}
Integrating over $u > 0$:
\begin{align}
    \label{eq:stochastic}
    \nonumber & \mathbb{E}[|\widehat{M}^{(n)} - M|]
    \; \! \! \leq \! \! \;
    \! \!  \int_0^\infty \! \! 2\exp\!\left(-\frac{2n\,q_{\min}^{(n)2}\,u^2}{(b-a)^2}
    \right)\mathrm{d}u
    \; \! \! = \\
    & =\;
    \frac{(b-a)}{q_{\min}^{(n)}\sqrt{2n}}.
\end{align}

\paragraph{Combining.}
Taking expectations in \cref{eq:decomp} and substituting
\cref{eq:bias,eq:stochastic} yields the main bound.

\paragraph{Case (i): Uniform sampling.}
If $q_i^{(n)} = n/N$ for all $i$, then $\sum_i q_i^{(n)}(M_i - M) =
(n/N)\sum_i(M_i - M) = 0$, so the bias is \emph{exactly} zero.
With $q_{\min}^{(n)} = n/N$, \cref{eq:stochastic} gives
$(b-a)N/(\sqrt{2}\,n^{3/2})$.

\paragraph{Case (ii): Uncertainty sampling.}
If $q_i^{(n)} = nu_i/U$, then $\|q^{(n)}\|_2 = (n/U)\|u\|_2$ and
$q_{\min}^{(n)} = nu_{\min}/U$. Substituting into \cref{eq:bias}:
\[
    \frac{\sqrt{N}}{n} \cdot \frac{n}{U}\|u\|_2 \cdot \sigma_M
    = \frac{\sqrt{N}\,\|u\|_2}{U}\,\sigma_M,
\]
which is independent of $n$. The stochastic term gives
$(b-a)U/(\sqrt{2}\,u_{\min}\,n^{3/2})$.
\end{proof}

\section{Additional Results}
\subsection{\at Detailed Strategies}
\label{app:at_details}

\paragraph{Surrogate} 

The Surrogate strategy selects test samples based on an auxiliary predictive model that estimates the uncertainty associated with each instance. The method first extracts embeddings for all unlabeled texts using the same pretrained encoder employed in the main evaluation pipeline and then trains a surrogate classifier on the resulting feature vectors paired with the ground-truth labels accumulated so far. Depending on the configuration, 2 alternative base learners can be used: a Random Forest with 300 trees and an RBF-kernel Support Vector Machine. In both cases, the base model is wrapped into \texttt{CalibratedClassifierCV} with five-fold internal calibration in order to reduce bias in predicted probabilities, which improves the stability of the entropy-based uncertainty measure.

Once the surrogate model is initialized, embeddings for all candidate samples are computed using the shared embedding extractor. The surrogate is then fitted on these embeddings and the current pool of labels. For each candidate sample, the surrogate outputs a vector of calibrated class probabilities. The uncertainty score is defined as the Shannon entropy of this distribution, computed as:
\[
s_i = -\sum_{c} p_{i,c} \log (p_{i,c} + \epsilon),
\]
where $\epsilon = 10^{-6}$ prevents numerical instabilities. Samples with the largest normalized scores are selected, corresponding to the most uncertain instances according to the surrogate model. 

\paragraph{Uncertainty}

The Uncertainty strategy selects test cases by estimating predictive uncertainty through Monte Carlo Dropout. The method supports two acquisition functions: mutual information (MI), which measures epistemic uncertainty over model predictions, and a Gaussian prior–based score derived from the variability of intermediate representations (GP). All samples are first tokenized and then passed to the underlying transformer model with dropout layers kept active by switching the model to training mode and ten forward passes are performed for each batch of samples. For mutual information, each forward pass produces a set of logits which are converted into probability distributions via softmax. The resulting tensor has shape $(M, B, C)$, where $M=50$ is the number of Monte Carlo samples, $B$ is the batch size, and $C$ is the number of classes. Uncertainty is computed as the difference between the entropy of the mean predictive distribution and the mean entropy of the individual Monte Carlo distributions. Formally, if $p_{i,c}^{(m)}$ denotes the predicted probability for class $c$ and sample $i$ at Monte Carlo iteration $m$, the mutual information is defined as:
\begin{align*}
\operatorname{MI}(i) =
-\sum_{c} \bar{p}_{i,c} \log(\bar{p}_{i,c} + \varepsilon)
\;+ \\
- \frac{1}{M} \sum_{m=1}^{M} 
\left[
\sum_{c} p_{i,c}^{(m)} \log(p_{i,c}^{(m)} + \varepsilon)
\right]
\end{align*}
with $\varepsilon = 10^{-6}$ to avoid numerical instability. 

When the Gaussian prior acquisition is selected, the model instead extracts mean-pooled last-layer embeddings at each Monte Carlo pass. Averaging these embeddings across $M$ iterations yields a mean representation for each sample, while the standard deviation across iterations provides a measure of uncertainty. The final acquisition score is computed as the squared sum of standard deviations across embedding dimensions, thus capturing the magnitude of variability in the latent space induced by dropout.

\paragraph{Coverage}

It aims at selecting test cases by maximizing the geometric spread of sampled points in the embedding space. This method chooses points that lie farthest from previously selected ones. All input texts are first transformed into dense vector representations using the pretrained multilingual encoder, following the same tokenization and embedding extraction pipeline used across all strategies. The resulting embedding matrix provides a geometric structure on which the coverage procedure operates.

The algorithm begins by selecting one initial point uniformly at random among all available embeddings. Given the set of selected indices $S$, let $x_{i}$ denote the embedding of the last selected point. For each candidate index $j \notin S$, the method computes the Euclidean distance:
\[
d(j) = \|\,x_j - x_{i}\,\|_2,
\]
which measures how far each unselected point lies from the most recently added sample. The core idea is to iteratively append the point achieving the maximal distance, thus ensuring that new selections explore embedding regions that have not yet been represented. Points already in $S$ are masked and assigned distance $-1$ to prevent re-selection.
This process continues until the desired number of samples is reached.
Because this strategy selects points explicitly at the periphery of previously covered regions, it ensures a broad geometric exploration of the embedding manifold without relying on clustering assumptions. The procedure is computationally efficient for moderate dataset sizes due to its incremental nature, and it is entirely deterministic except for the choice of the initial random point.

\paragraph{Agreement} The Agreement strategy ranks unlabeled samples according to the level of disagreement among attention heads in a multi-head attention module applied to BERT representations. All text inputs are encoded with the pretrained BERT model from HuggingFace, using the corresponding tokenizer. For each batch, the tokenizer outputs PyTorch tensors that are forwarded through the BERT encoder. The final hidden states are then processed by a lightweight neural module consisting of an 8-head self-attention layer. The attention module receives inputs in the shape  and returns both transformed token representations and attention weight matrices. For each sample, the uncertainty score is computed by taking the variance of the attention weights across the eight heads. More precisely, for each token position $t$, the tensor containing all token-to-token attention scores for head $h$ is considered, and its variance across heads is computed. The resulting variance matrices are averaged over all token positions, producing a single scalar score for each sample. All unlabeled samples not previously selected by the algorithm are ranked in descending order according to their normalized attention-based scores.

\subsection{\at vs Active Learning}
\label{app:al_vs_at}

As pointed out by \citet{kossen2021active}, while acquisition functions for Active Learning have been extensively studied, they cannot be directly applied to \at due to fundamental conceptual differences, as \at does not involve using data to train a model.

There are three key distinctions:
\begin{itemize}
\item Conventional Active Learning methods tend to avoid regions with high aleatoric (inherent data) uncertainty while prioritizing areas with high epistemic (model) uncertainty. This principle underlies acquisition functions such as BALD \citep{houlsby2011bayesian} and BatchBALD \citep{kirsch2019batchbald}. In contrast, for \at, regions characterized by high aleatoric uncertainty may be essential for accurate estimation.

\item \citet{imberg2020optimal} demonstrate that the optimal acquisition strategy for Active Learning aims to minimize expected generalization error after training concludes. They illustrate how this objective necessitates additional considerations beyond simply reducing the variance of the loss estimator.

\item \citet{bias} establish that biased loss estimators can actually benefit the training process, as they frequently offset the inherent bias present in the training loss. However, this advantage disappears during testing, where minimizing bias becomes the primary objective.
\end{itemize}

A detailed plot containing the results on all the \at strategies can be found in~\cref{fig:additional}

\subsection{Impact of Embedding Strategy and Predictor}
\label{app:embedding}

\begin{figure*}[t]
\begin{subfigure}{0.48\textwidth}
    \centering
    \includegraphics[width=\textwidth]{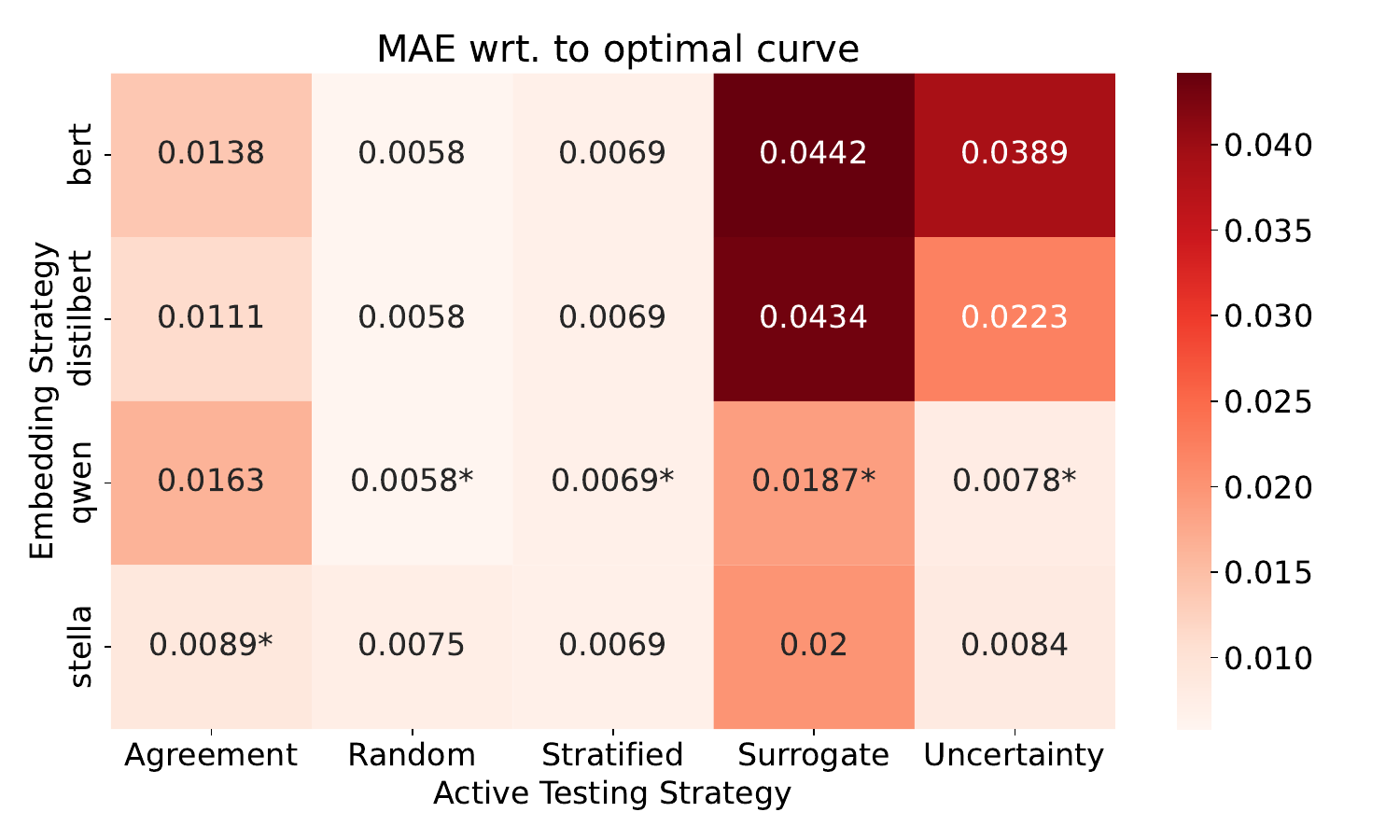}
    \end{subfigure} 
    \begin{subfigure}{0.48\textwidth}
    \centering
    \includegraphics[width=\textwidth]{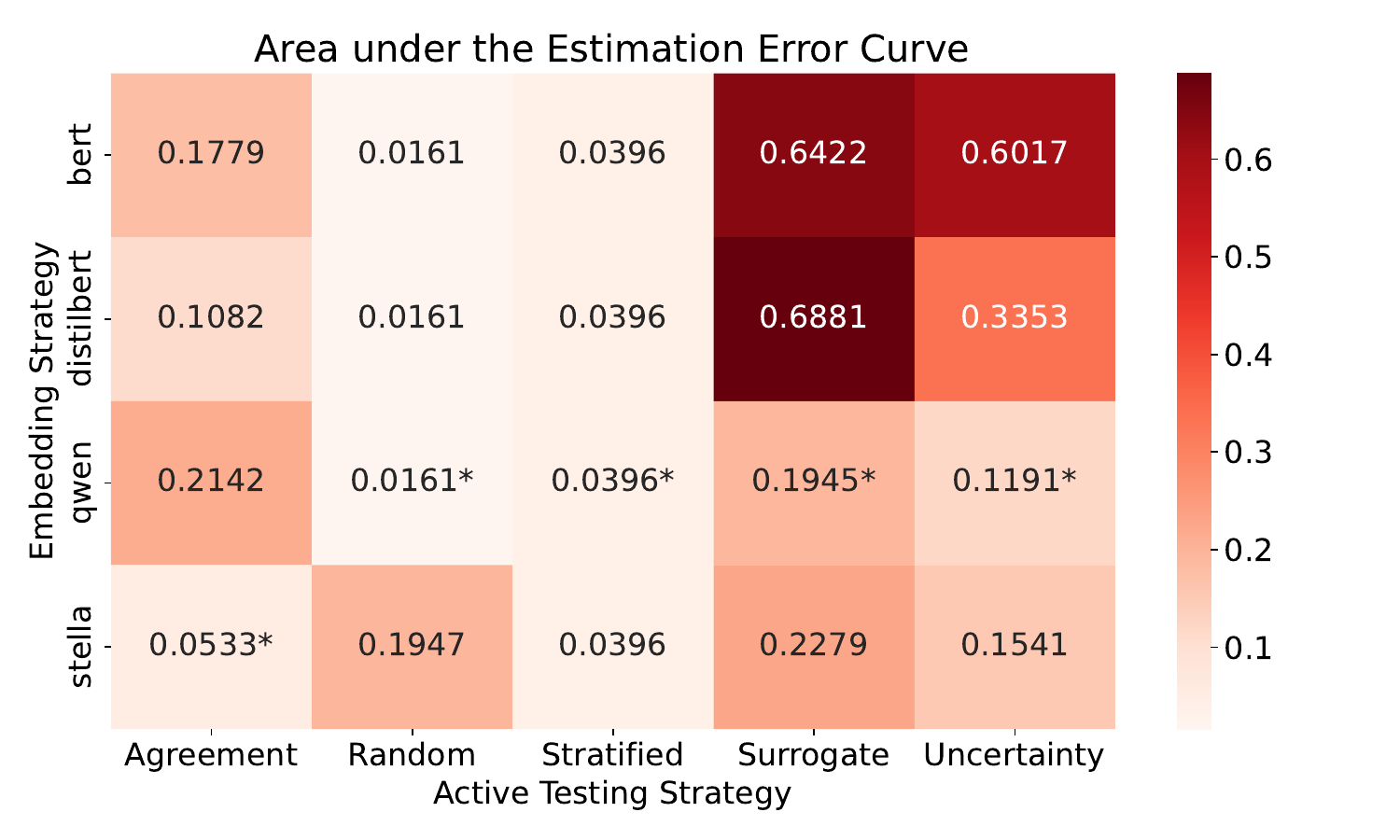}
    \end{subfigure} \\
\begin{subfigure}{0.48\textwidth}
    \centering
    \includegraphics[width=\textwidth]{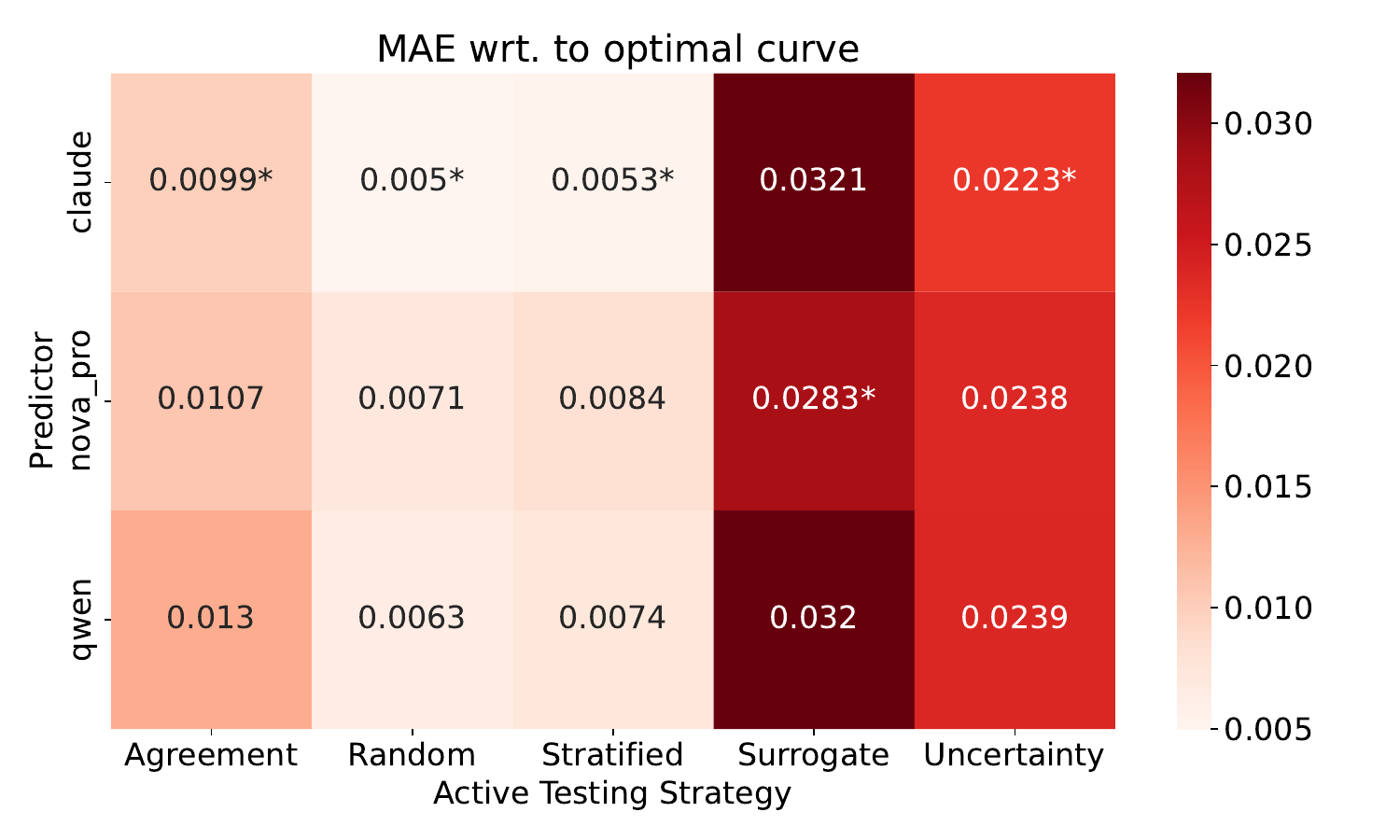}
\end{subfigure}
    \begin{subfigure}{0.48\textwidth}
    \centering
    \includegraphics[width=\textwidth]{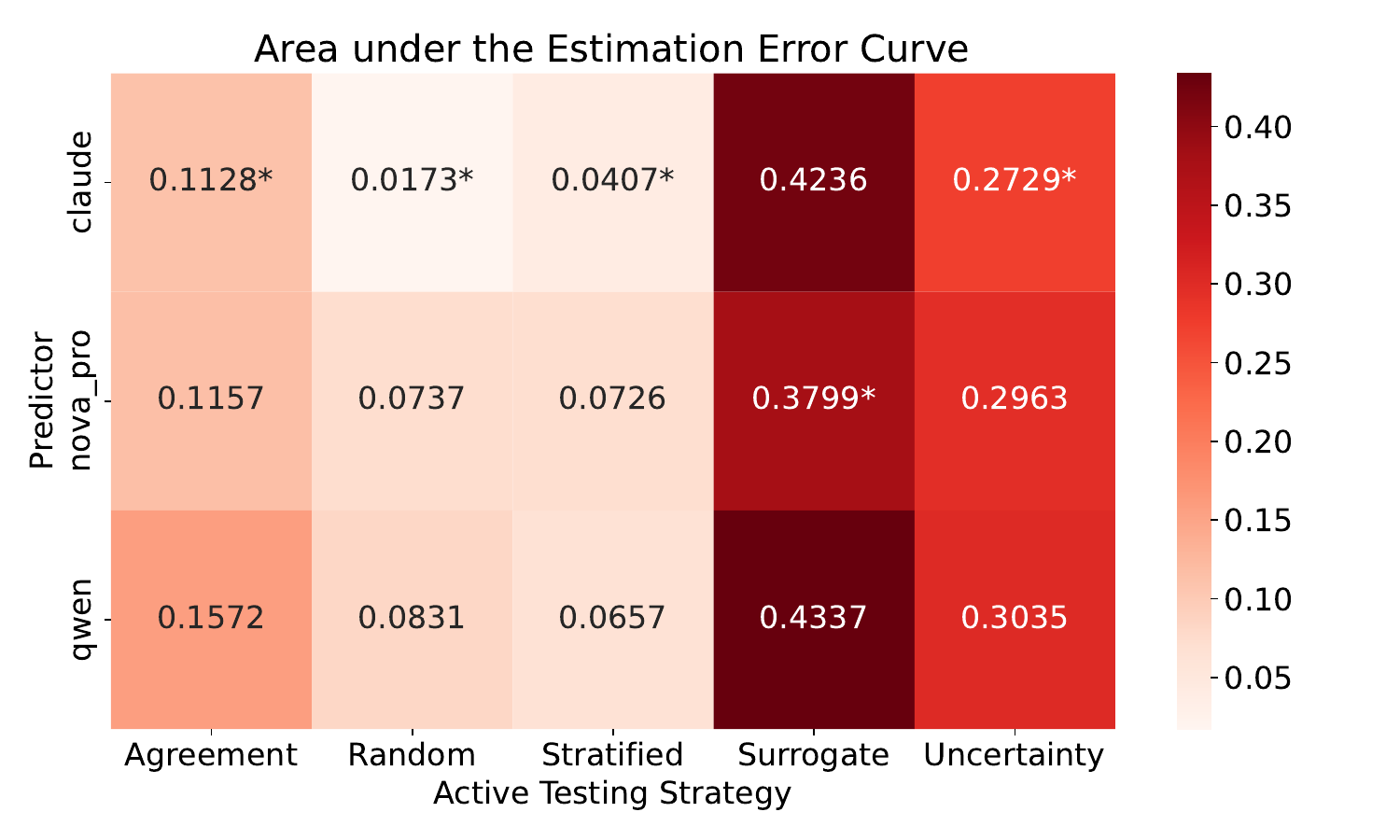}
\end{subfigure}
    \caption{Confusion matrices with the \at Strategies on the x-axis. The top figures have the embedding strategies on the y-axis while the bottom ones the predictors. The best results are marked with *. MAE wrt. optimal curve on the left and Area Under the Estimation Error Curve on the right. $\downarrow$ is better.}
    \label{fig:predictor_difference_confusion}
\end{figure*}

We investigate how the choice of embedding strategies and predictor affect the performance using two metrics: Area under the Estimation Error Curve (normalized by budget; lower=faster error reduction), and Mean Absolute Error (MAE) from the optimal curve (lower=closer to best achievable performance), averaged across datasets.

In \cref{fig:predictor_difference_confusion}, Qwen achieves the best performance across embedding strategies, while Claude performs well across predictors. Both analyses reveal the same pattern: Random and Stratified Random show comparable performance regardless of embedder or predictor, while more sophisticated \at strategies (Agreement, Surrogate) exhibit larger performance gaps. This suggests that both embedding and predictor quality become critical as the \at strategy grows more complex, since these methods rely on the geometric structure of the embedding space and model-specific signals to identify informative samples.

\subsection{Predictors' classification performance}
\label{app:predictor_performance}
\cref{app:metrics_total} shows the performance of the predictors in terms of classification, while \cref{app:metrics_summarization} in terms of summarization. 
For the classification task, the prompt we use is the following:

\begin{tcolorbox}[colback=gray!10, colframe=gray!50,
width=\columnwidth]
\small{
\label{prompt:time_dependent}
{\ttfamily This is a text classification task. The possible labels are [...] while the indices of the labels are [...].
Output only 'Label: the predicted index of the label that you predict'. This is the sentence to classify: [...]. 
}
}
\end{tcolorbox}

For the summarization task, the prompt we use is the following:

\begin{tcolorbox}[colback=gray!10, colframe=gray!50,
width=\columnwidth]
\small{
{\ttfamily This is a text summarization task. This is the sentence that must be summarized: [...]
Output only the summary. 
}
}
\end{tcolorbox}

\begin{table*}[h]
    \centering
  \resizebox{\linewidth}{!}{
  \begin{tabular}{lccc|ccc|ccc}
  \toprule
 & \multicolumn{3}{c|}{Claude} & \multicolumn{3}{c|}{Nova Pro} & \multicolumn{3}{c}{Qwen} \\
 Dataset & Accuracy & Recall & F1 & Accuracy & Recall & F1 & Accuracy & Recall & F1 \\
\midrule
        \href{https://huggingface.co/datasets/fancyzhx/ag_news}{AG's News} & \textbf{0.8814} & \textbf{0.8814} & \textbf{0.8813} & \underline{0.8082} & \underline{0.8082} & \underline{0.8066} & 0.7863 & 0.7863 & 0.7886 \\
        \href{https://huggingface.co/datasets/PolyAI/banking77}{Banking77} & \textbf{0.7925} & \textbf{0.7925} & \textbf{0.7996} & \underline{0.6942} & \underline{0.6942} & \underline{0.7072} & 0.6581 & 0.6581 & 0.6781 \\
        \href{https://huggingface.co/datasets/pietrolesci/dbpedia_14_indexed}{DBPedia} & \textbf{0.9822} & \textbf{0.9822} & \textbf{0.9822} & \underline{0.9385} & \underline{0.9385} & \underline{0.9401} & 0.9312 & 0.9312 & 0.9317 \\
        \href{https://huggingface.co/datasets/nid989/FNC-1}{FNC-1} & \textbf{0.4806} & \textbf{0.4806} & \textbf{0.4462} & 0.2933 & 0.2933 & 0.1961 & \underline{0.4082} & \underline{0.4082} & \underline{0.3406} \\
        % \href{https://huggingface.co/datasets/nyu-mll/glue}{MNLI} & M & M & M\\
        \href{https://huggingface.co/datasets/nyu-mll/glue}{QNLI} & \textbf{0.5072} & \textbf{0.5072} & \textbf{0.5988} & \underline{0.5017} & \underline{0.5017} & \underline{0.5484} & 0.4882 & 0.4882 & 0.5033 \\
        \href{https://huggingface.co/datasets/stanfordnlp/sst2}{SST-2} & \textbf{0.9564} & \textbf{0.9564} & \textbf{0.9564} & 0.9255 & 0.9255 & 0.9255 & \underline{0.9427} & \underline{0.9427} & \underline{0.9427} \\
        \href{https://huggingface.co/datasets/OxAISH-AL-LLM/trec6}{TREC-6} & \underline{0.7980} & \underline{0.7980} & \underline{0.8009} & 0.7440 & 0.7440 & 0.7395 & \textbf{0.9465} & \textbf{0.9465} & \textbf{0.9465} \\
         \href{https://huggingface.co/datasets/stanfordnlp/imdb}{IMDB}$^*$ & \textbf{0.9621} & \textbf{0.9621} & \textbf{0.9621} & \underline{0.9519} & \underline{0.9519} & \underline{0.9519} & 0.9465 & 0.9465 & 0.9465  \\
         \href{https://huggingface.co/datasets/pietrolesci/pubmed-20k-rct}{PubMed}$^*$ & \textbf{0.7202} & \textbf{0.7202} & \textbf{0.7554} & \underline{0.6453} & \underline{0.6453} & \underline{0.6826} & 0.6391 & 0.6391 & 0.6645 \\
         \href{https://huggingface.co/datasets/dair-ai/emotion}{Emotion}$^*$ & 0.5830 & 0.5830 & 0.5776 & \underline{0.5695} & \underline{0.5695} & \underline{0.5630} & \textbf{0.5830} & \textbf{0.5830} & \textbf{0.5784} \\
         \href{https://huggingface.co/datasets/cornell-movie-review-data/rotten_tomatoes}{Rotten}$^*$ & \textbf{0.9343} & \textbf{0.9343} & \textbf{0.9344} & 0.8987 & 0.8987 & 0.8986 & \underline{0.9071} & \underline{0.9071} & \underline{0.9071} \\
         \href{https://huggingface.co/datasets/tyqiangz/multilingual-sentiments}{Multilingual}$^*$ & \textbf{0.8069} & \textbf{0.8069} & \textbf{0.8081} & \underline{0.7494} & \underline{0.7494} & \underline{0.7502} & 0.7046 & 0.7046 & 0.7103 \\
         \bottomrule
  \end{tabular} 
  }
  \caption{Text classification metrics on the full datasets. For each metric and dataset, the best predictor is \textbf{bolded} while the second best is \underline{underlined}.}
    \label{app:metrics_total}
\end{table*}

\begin{table*}[h]
    \centering
  \begin{tabular}{lccc}
  \toprule
 Dataset & Claude & Nova Pro & Qwen \\
\midrule
         \href{https://huggingface.co/datasets/abisee/cnn_dailymail}{CNN}$^*$ & 0.5074 & \textbf{0.5078} & 0.4827 \\
         \href{https://huggingface.co/datasets/csebuetnlp/xlsum}{XLSum}$^*$ & 0.4427 & \textbf{0.4543} & 0.4620\\
         \bottomrule
  \end{tabular} 
  \caption{Summarization results on the full datasets in terms of ROUGE-1. For each metric and dataset, the best predictor is \textbf{bolded} while the second best is \underline{underlined}.}
    \label{app:metrics_summarization}
\end{table*}

\begin{figure*}[t]
    \centering
    \includegraphics[width=\linewidth]{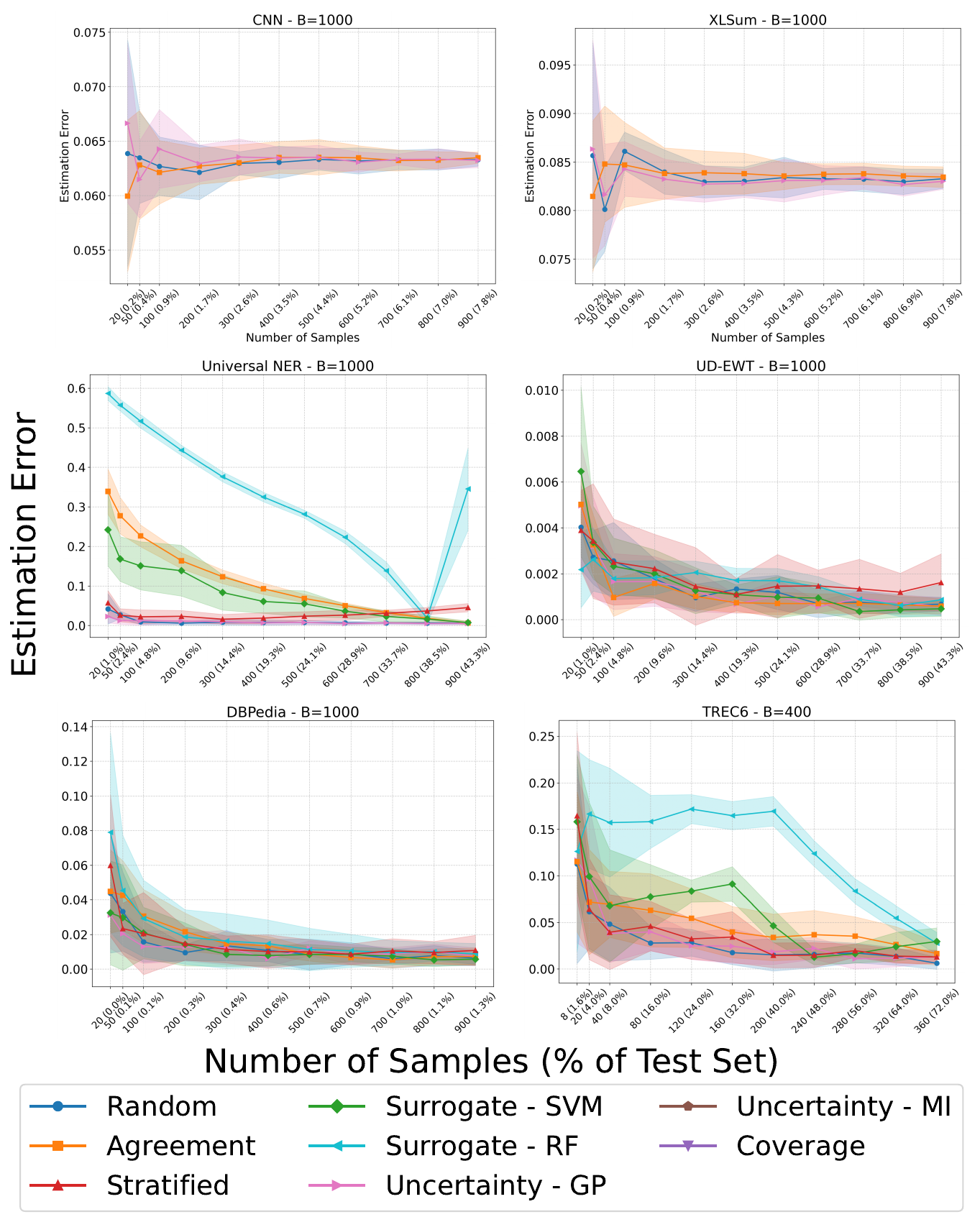}
    \caption{\at results in terms of accuracy on all the datasets by fixing the predictor (Qwen for classification, DeBERTa finetuned for POS tagging and CNER for NER and the embedding strategy (Qwen).}
    \label{fig:additional}
\end{figure*}

\end{document}